\documentclass[runningheads]{llncs}
\usepackage{graphicx}
\usepackage{hyperref}
\usepackage{xcolor}
\usepackage{todonotes}
\usepackage{verbatim}
\usepackage{makecell}
\renewcommand\UrlFont{\color{blue}\rmfamily}
\usepackage{multirow,rotating}
\usepackage{booktabs}
\usepackage{algorithm,algpseudocode,algorithmicx}
\usepackage{amsmath,amssymb,bm}
\usepackage[makeroom]{cancel}
\usepackage{xfrac}

\newcommand{\norm}[1]{\bigl\lVert#1\bigl\rVert}

\newcommand{\ug}{\mathbf{u}}
\newcommand{\w}{\mathbf{w}}
\newcommand{\ub}{\bar{\mathbf{u}}}
\newcommand{\ut}{\tilde{\mathbf{u}}}
\newcommand{\pdf}{e^{-\frac{1}{2}\lVert\ug\rVert^2}}
\newcommand{\ab}{\mathbf{a}}
\newcommand{\bb}{\mathbf{b}}
\newcommand{\vb}{\mathbf{v}}
\newcommand{\x}{\mathbf{x}}
\newcommand{\y}{\mathbf{y}}
\newcommand{\m}{\mathbf{m}}

\newcommand{\X}{\mathcal{X}}
\newcommand{\E}{\mathbb{E}}
\newcommand{\R}{\mathbb{R}}
\newcommand{\V}{\mathbb{V}}
\newcommand{\N}{\mathcal{N}}
\newcommand{\p}{\mathrm{pdf}}
\newcommand{\bigO}{\mathcal{O}}

\begin{document}
\title{Sparse Perturbations for Improved Convergence in Stochastic Zeroth-Order Optimization
}
\titlerunning{Sparse Perturbations for Improved Convergence in SZO Optimization}
%
\author{Mayumi Ohta\inst{1}\and
Nathaniel Berger\inst{1} \and
Artem Sokolov\inst{1}\and
Stefan Riezler\inst{1,2}}
\authorrunning{M. Ohta et al.}
%
\institute{Department of Computational Linguistics \and Interdisciplinary Center for Scientific Computing (IWR)\\
Heidelberg University\\
Im Neuenheimer Feld 325\\
69120 Heidelberg, Germany \\
\email{\{ohta,berger,sokolov,riezler\}@cl.uni-heidelberg.de}\\
\url{https://www.cl.uni-heidelberg.de/statnlpgroup/}}
\maketitle              
%
\begin{abstract}
Interest in stochastic zeroth-order (SZO) methods has recently been revived in black-box optimization scenarios such as adversarial black-box attacks to deep neural networks. SZO methods only require the ability to evaluate the objective function at random input points, however, their weakness is the dependency of their convergence speed on the dimensionality of the function to be evaluated. We present a sparse SZO optimization method that reduces this factor to the expected dimensionality of the random perturbation during learning. We give a proof that justifies this reduction for sparse SZO optimization for non-convex functions. Furthermore, we present experimental results for neural networks on MNIST and CIFAR that show empirical sparsity of true gradients, and faster convergence in training loss and test accuracy and a smaller distance of the gradient approximation to the true gradient in sparse SZO compared to dense SZO. 

\end{abstract}
%
%
\section{Introduction}

Zeroth-order optimization methods have gained renewed interest for solving machine learning problems where only the zeroth-order oracle, i.e., the value of the objective function but no explicit gradient, is available. Recent examples include black-box attacks on deep neural networks where adversarial images that lead to misclassification are found by approximating the gradient through a comparison of function values at random perturbations of input images \cite{ChenETAL:17}. The advantage of simple random search for scalable and reproducible gradient-free optimization has also been recognized in reinforcement learning \cite{ManiaETAL:18,SalimansETAL:17} and hyperparameter tuning for deep neural networks \cite{EbrahimiETAL:17}, and it is a mainstay in optimization of black-box systems and in simulation optimization \cite{Fu:06,NesterovSpokoiny:15}. 
While zeroth-order optimization applies in principle even to non-differentiable functions, in practice Lipschitz-smoothness of the black-box function being evaluated can be assumed. This allows to prove convergence for various zeroth-order gradient approximations \cite{AgarwalETAL:10,DuchiETAL:15,FlaxmanETAL:05,GhadimiLan:12,NesterovSpokoiny:15,Shamir:17,YueJoachims:09}. However, even in the optimal case, zeroth-order optimization of $n$-dimensional functions suffers a factor of $\sqrt{n}$ in convergence rate compared to first-order gradient-based optimization. 
The goal of our paper is to show theoretically and practically that \emph{sparse perturbations} in stochastic zeroth-order (SZO) optimization can improve convergence speed considerably by replacing the dependency on the dimensionality of the objective function by a dependency on the expected dimensionality of the random perturbation vector. This approach can be motivated by an observation of empirical sparsity of gradients in our experiments on neural networks, or by natural sparsity of gradients in applications to linear models with sparse input features \cite{SokolovETAL:18}. 
We give a general convergence proof for non-convex stochastic zeroth-order optimization for Lipschitz-smooth functions that is independent of the dimensionality reduction schedule applied, and shows possible linear improvements in iteration complexity. Our proof is based on \cite{NesterovSpokoiny:15} and fills in the necessary gaps to verify the dependency of convergence speed on the dimensionality of the perturbation instead of on the full functional dimensionality.
We present experiments that perform dimensionality reduction on the random perturbation vector by iteratively selecting parameters with high magnitude for further SZO tuning, and freezing other parameters at their current values. Another dimensionality reduction technique selects random masks by the heldout performance of selected sub-networks, and once a sub-network architecture is identified, it is further fine-tuned by SZO optimization. 
In our experiments, we purposely choose an application that allows optimization with standard first-order gradient-based techniques, in order to show improved gradient approximation by our sparse SZO technique compared to standard SZO with full perturbations. Our experimental results confirm a smaller distance to the true sparse gradient for the gradient obtained by sparse SZO, and improved convergence in training loss and test accuracy. 
Furthermore, we compare our technique to a zeroth-order version of iterative magnitude pruning \cite{FrankleCarbin:19,HanETAL:15}. This technique zeros-out unimportant parts of the weight vector while our proposed technique zeros-out only perturbations and freezes unimportant parts of the weight vector at their current values. We find similar convergence speed, but a strong overtraining effect on the test set if weights are pruned to zero values. An estimation of the local Lipschitz constant for the trained models shows that it is growing in the number of iterations for the pruning approach, indicating that zeroing-out low magnitude weights may lead the optimization procedure in a less smooth region of the search space. We conjecture that a search path with a small local Lipschitz constant might have a similar desirable effect as flat minima \cite{KeskarETAL:17} in terms of generalization on unseen test data. Our code is publicly available\footnote{\UrlFont{https://github.com/StatNLP/sparse\_szo}}.

\section{Related Work}

SZO techniques in optimization date back to the finite-difference method for gradient estimation of \cite{KieferWolfowitz:52} where the value of each component of a weight vector is perturbed separately while holding the other components at nominal value. This technique has since been replaced by more efficient methods based on simultaneous perturbation of all weight vector components \cite{KushnerYin:03,Spall:92,Spall:03}. The central idea of these approaches can be described as approximating non-differentiable functions by smoothing techniques, and applying first-order optimization to the smoothed function \cite{FlaxmanETAL:05}. Several works have investigated different update rules and shown the advantages of two-point or multi-point feedback for improved convergence speed \cite{AgarwalETAL:10,DuchiETAL:15,GhadimiLan:12,LiuETAL:18,NesterovSpokoiny:15,Shamir:17,YueJoachims:09}. Recent works have investigated sparsity methods for improved convergence in high dimensions \cite{BalasubramanianGhadimi:18,WangETAL:18}. These works have to make strong assumptions of function sparsity or gradient sparsity. A precursor to our work that applies sparse SZO to linear models has been presented by \cite{SokolovETAL:18}.

Connections of SZO methods to evolutionary algorithms and reinforcement learning have first been described in \cite{Spall:03}. Recent work has applied SZO techniques successfully to reinforcement learning and hyperparameter search for deep neural networks \cite{EbrahimiETAL:17,ManiaETAL:18,PlappertETAL:18,SalimansETAL:17,SehnkeETAL:10}. Similar to these works, in our experiments we apply SZO techniques to Lipschitz-smooth deep neural networks. 

Recent practical applications of SZO techniques have been presented in the context of adversarial black-box attacks on deep neural networks. Here a classification function is evaluated at random perturbations of input images with the goal of efficiently finding images that lead to misclassification  \cite{ChenETAL:17,ChengETAL:19}. Because of the high dimensionality of the adversarial attack space in image classification, heuristic methods to dimensionality reduction of perturbations in SZO optimization have already put to practice in \cite{ChenETAL:17}. Our approach presents a theoretical foundation for these heuristics.

\section{Sparse Perturbations in SZO Optimization for Nonconvex Objectives}

We study a stochastic optimization problem of the form
\begin{align}
 \min_\w f(\w), \text{ where } f(\w) := \E_\x[F(\w, \x)],
\end{align}
and where $\E_\x$ denotes the expectation over inputs $\x \in \X$, and $\w \in \R^n$ parameterizes the objective function $F$. 
We address the case of non-convex functions $F$ for which we assume Lipschitz-smoothness\footnote{$C^{p,k}$ denotes the class of $p$ times differentiable functions whose $k$-th derivative  is Lipschitz continuous.}, i.e., 
$F(\w,\x) \in C^{1,1} $ is Lipschitz-smooth $ \mathrm{iff} \; \forall \w, \w', \x: 
  \norm{\nabla F(\w,\x) - \nabla F(\w',\x)} \leq L(F) \norm{\w - \w'}$, 
where $\norm{\cdot}$ is the L2-norm and $L(F)$ denotes the Lipschitz constant of $F$. This condition is equivalent to 
\begin{align}\label{def:lipschitz}
  \lvert F(\w') - F(\w) - \langle \nabla F(\w), \w' - \w\rangle\rvert \leq \frac{L(F)}{2} \norm{\w - \w'}^2.
\end{align}
It directly follows that $f\in C^{1,1}$ if $F\in C^{1,1}$.

\cite{NesterovSpokoiny:15} show how to achieve a smooth version of an arbitrary function $f(\w)$ by Gaussian blurring that assures continuous derivatives everywhere in its domain. In their work, random perturbation of parameters is based on sampling a $n$-dimensional Gaussian random vector $\ug$ from a zero-mean isotropic multivariate Gaussian with unit $n \times n$ covariance matrix $\Sigma = \bm{I}$. The probability density function $\p(\ug)$ is defined by
\begin{align}
  \mathbf{u} &\sim \N(\bm{0}, \Sigma),\\
  \p(\mathbf{u}) &:=\frac{1}{\sqrt{(2\pi)^n\cdot \det\Sigma}} e^{-\frac{1}{2}\mathbf{u}^\top\Sigma^{-1}\mathbf{u}}.
\end{align}
A Gaussian approximation of a function $f$ is then obtained by the expectation over perturbations $\E_{\ug}[f(\w + \mu \ug)]$, where $\mu > 0$ is a smoothing parameter.
Furthermore, a Lipschitz-continuous gradient can be derived even for a non-differentiable original function $f$ by applying standard differentiation rules to the Gaussian approximation, yielding $\E_{\ug}[\frac{f(\w + \mu \ug) - f(\w)}{\mu}\ug]$.

The central contribution of our work is to show how a sparsification of the random perturbation vector $\ug$ directly enters into improved bounds on iteration complexity. This is motivated by an observation of empirical sparsity of gradients in our experiments, as illustrated in the supplementary material (Figures \ref{fig:sparsity-cifar} and \ref{fig:sparsity-mnist}). Concrete techniques for sparsification will be discussed as parts of our algorithm (Section \ref{sec:algo}), however, our theoretical analysis applies to any method for sparse perturbation. 
Let $\m^{(t)} \in \{0,1\}^{n}$ be a mask specified at iteration $t$, and let $\odot$ denote the componentwise multiplication operator. A sparsified version $\bar{\vb}^{(t)}$ of a vector $\vb^{(t)} \in \R^{n}$ is gotten by
\begin{align}
  \bar{\vb}^{(t)} := \m^{(t)} \odot \vb^{(t)}. 
\end{align}
The number of nonzero parameters of $\bar{\vb}^{(t)}$ is defined for a given mask $\m^{(t)}$ as
$\bar{n}^{(t)} := \lVert \m^{(t)} \rVert.$  
Given the definition of a sparsification mask stated above, we can define a sparsified Gaussian random vector drawn from $\N(\bm{0}, \bm{I})$ by $\ub := \m \odot \ug$. Using this notion of sparse perturbation, we redefine the smooth approximation of $f$, with smoothing parameter $\mu > 0$, as
\begin{align}
 f_\mu(\w) :=& \E_{\ub}\left[ f(\w+\mu\ub) \right] .
\end{align}
Note that $f_\mu$ is Lipschitz-smooth with $L(f_\mu) < L(f)$.

\section{Convergence Analysis for Sparse SZO Optimization}

\cite{NesterovSpokoiny:15} show that for Lipschitz-smooth functions $F$, the distance of the gradient approximation to the true gradient can be bounded by the Lipschitz constant and by the norm of the random perturbation. The term $\E_{\ug}[\norm{\ug}^p]$ can itself be bounded by a function of the exponent $p$ and the dimensionality $n$ of the function space. This is how the dependency on $n$ enters iteration complexity bounds and where an adaptation to sparse perturbations enter the picture. The simple case of the squared norm of the random perturbation given below illustrates the idea. For $p=2$, we have
\begin{align}
\E_{\ub}\left[\norm{\ub}^2\right] &= \E_{\ub}\left[\bar{u}_1^2+\bar{u}_2^2+\cdots+\bar{u}_n^2\right] \nonumber\\
&= \E_{\ub}\left[\bar{u}_1^2\right]+\E_{\ub}\left[\bar{u}_2^2\right]+\cdots+\E_{\ub}\left[\bar{u}_n^2\right] \nonumber\\
&= \V_{\ub}\left[\bar{u}_1\right]+\V_{\ub}\left[\bar{u}_2\right]+\cdots+\V_{\ub}\left[\bar{u}_n\right] \nonumber\\
&= \bar{n}. \label{eq:lemma1_p_2}
\end{align}
Intuitively this means that if a coordinate $i$ in the parameter space is not perturbed, no variance $\V_{\ub}[\bar{u}_i]$ is incurred. The smaller the variance, the smaller the factor $\bar{n}$ that directly influences iteration complexity bounds. 
The central Lemma \ref{lemma1} gives a bound on the expected norm of the random perturbations for the general case of $p \geq 2$: Intuitively, if several coordinates are masked and thus not perturbed, the determinant of the covariance matrix reduces to a product of variances of the unmasked parameters. This allows us to bound the perturbation factor for each input by $\bar{n} \ll n$ where the sparsity pattern is given by the masking strategy determined in the algorithm (see Section~\ref{sec:algo}).
\begin{lemma}\label{lemma1}
Let $\ub = \m \odot \mathbf{u}$ with $\mathbf{u} \sim \N(\bm{0}, \bm{I})$, then
  \begin{align}
    \E_{\ub}\left[\norm{\ub}^p\right] \leq (\bar{n} + p)^{\sfrac{p}{2}} ~ \text{for} \; p \geq 2. \label{eq:lemma1}
    \end{align}
  \end{lemma}
Given the reduction of the perturbation factor to $\bar{n} \ll n$, we can use the reduced factor in further Lemmata and in the main Theorem that shows improved iteration complexity of sparse SZO optimization. In the following we will state the Lemmata and Theorems that extend \cite{NesterovSpokoiny:15} to the case of sparse perturbations. For completeness, full proofs of all Lemmata and Theorems can be found in the supplementary material.

Lemma \ref{lemma2} applies standard differentiation rules to the function $f_\mu(\w)$, yielding a Lipschitz-continuous gradient. 

\begin{lemma}\label{lemma2}
Let $F\in C^{1,1}$ and $\ub = \m \odot \mathbf{u}$ with  $\mathbf{u} \sim \N(\bm{0}, \bm{I})$. Then we have
  \begin{align}
    \nabla f_\mu(\w) & = \E_{\ub}[\frac{f(\w + \mu \ub)}{\mu}\ub] \label{eq:1-point}\\ 
    & = \E_{\ub}[\frac{f(\w + \mu \ub) - f(\w)}{\mu}\ub] \label{eq:2-point}\\ 
    & = \E_{\ub}[\frac{f(\w + \mu \ub) - f(\w - \mu \ub)}{2\mu}\ub]. \label{eq:2-sided}
    \end{align}
  \end{lemma}

In the rest of this paper we define our update rule as
\begin{align}\label{def:g_mu}
g_\mu(\w) := \frac{f(\w+\mu\ub) - f(\w)}{\mu}\ub
\end{align}
from (\ref{eq:2-point}).
Lemma \ref{lemma2} implies that $g_\mu(\w)$ is an unbiased estimator of $\nabla f_\mu(\w)$.

The next Lemma shows how to bound the distance of the gradient approximation to the true gradient in terms of $\mu$, $L(f)$, and the number of perturbations~$\bar{n}$.

\begin{lemma}\label{lemma3}
Let $f\in C^{1,1}$ Lipschitz-smooth, then
  \begin{align}
    \lVert \nabla f(\w) \rVert^2 &\leq 2 \lVert \nabla f_\mu(\w) \rVert^2 + \frac{\mu^2L^2(f)}{2} (\bar{n}+4)^{3}. \label{eq:lemma3}
    \end{align}
  \end{lemma}
  
The central Theorem shows that the iteration complexity of an SZO algorithm based on update rule \eqref{def:g_mu} is bounded by the Lipschitz constant $L$ and the expected dimensionality $\hat{n}$ of sparse perturbations with respect to iterations: 
\begin{theorem}\label{theorem1}
Assume a sequence $\{\w^{(t)}\}_{t\geq 0}$ be generated by Algorithm \ref{algorithm1}. Let $\mathcal{X} = \{\x^{(t)}\}_{t\geq 0}$ and $ \bar{\mathcal{U}} = \{\ub^{(t)}\}_{t\geq 0}$. Furthermore, let $f_\mu^\star$ denote a stationary point such that $f_\mu^\star \leq f_\mu(w^{(t)}) ~ \forall \; t>0$  and $\hat{n} \geq \E_{t}[\bar{n}^
{(t)}] := \frac{1}{T+1}\sum_{t=0}^{T+1}\bar{n}^
{(t)}$ the upper bound of the expected number of nonzero entries of~$\bar{\mathcal{U}}$. Then, choosing learning rate $\hat{h} = \frac{1}{4(\hat{n}+4)L}$ and $\mu = \Omega{\left(\dfrac{\epsilon}{\hat{n}^{\sfrac{3}{2}}L}\right)}$ where $L(f)\leq L$ for all $f(\w^{(t)})$, we have
  \begin{align}
    \E_{\bar{\mathcal{U}},\mathcal{X}}\left[\norm{\nabla f(\w^{(T)})}^2\right] \leq \epsilon^2
  \end{align}
for any $T = \bigO{\left(\frac{\hat{n}L}{\epsilon^2}\right)}$.
\end{theorem}

Note that $\hat{n} \geq \E_{t}\left[\bar{n}^{(t)}\right]$ is much smaller than $n = \E_{t}\left[n^{(t)}\right] = \frac{1}{T+1}\sum_{t=0}^{T+1} n$, if $\bar{n}^{(t)}$ is chosen such that $\bar{n}^{(t)} \ll n$ at each iteration $t>0$ of Algorithm \ref{algorithm1}. This Theorem implies that both strong sparsity patterns and the smoothness of the objective function can boost the convergence speed up to linear factor.

\begin{table}[t]
    \begin{center}
    \begin{tabular}{ccc}
        \toprule
         dataset & MNIST & CIFAR-10 \\
         \midrule
         trainset size & 50,000 & 40,000\\
         devset size & 10,000 & 10,000\\
         testset size & 10,000 & 10,000\\
         architecture & feed-forward NN & CNN\\
         model size & 266,610 & 4,301,642\\
         batch size & 64  & 64 \\
         learning rate $h$ & 0.2 & 0.1  \\
         smoothing $\mu$ & 0.05  & 0.05  \\
         variance & 1.0  & 1.0 \\
         total epochs & 100 & 100\\
         sparsification interval & 5 epochs & 5 epochs \\
         \bottomrule
    \end{tabular}
    \end{center}
    \caption{Data statistics and system settings.
    }
    \label{table:hyperparameters}
\end{table}

\section{Algorithms for Sparse SZO Optimization}
\label{sec:algo}

\begin{algorithm}[t!]
  \caption{Sparse SZO Optimization}\label{algorithm1}
  \begin{algorithmic}[1]
    \State INPUT: dataset $\X=\{\x_0, \cdots, \x_T\}$, sequence of learning rates $h$, sparsification interval $s$, number of samples $k$, smoothing parameter $\mu>0$ 
    \State Initialize mask: $\m^{(0)} = \bm{1}^n$ \hfill ($\bar{n}^{(0)} = n$)
    \label{line:init_mask}
    \State Initialize weights: $\w^{(0)}$
    \For{$t=0,\cdots, T$} 
    \If{$\mathrm{mod}(t,s) = 0$}
      \State Update mask $\m^{(t')} = \mathrm{get\_mask}(\w^{(t)})$ \label{line:get_mask}
      \If{pruning}
      \State  $\w^{(t)} = \m^{(t')} \odot \w^{(t)}$\label{line:set_weights}
      \label{line:mask_weights}
      \EndIf
    \EndIf
    \State Observe $\x^{(t)} \in \X$
    \For{$j=1, \cdots, k$}
    
    \State Sample unit vector $\mathbf{u}^{(t)} \sim \N(\bm{0},\bm{I})$
    \State Apply mask $\ub^{(t)} = \m^{(t')} \odot \mathbf{u}^{(t)}$ \hfill ($\bar{n}^{(t)} \ll n$)
    \label{line:mask_perturbations}
    \State Compute $g_\mu^{(j)}(\w^{(t)})$ 
    \EndFor
    \State Average $g_\mu(\w^{(t)}) = \underset{j}{\mathrm{avg}}(g_\mu^{(j)}(\w^{(t)}))$
    \State Update $\w^{(t+1)} = \w^{(t)} - h^{(t)}g_\mu(\w^{(t)})$
    \EndFor
    \State OUTPUT: sequence of $\{\w^{(t)}\}_{t\geq 0}$
  \end{algorithmic}
\end{algorithm}

\begin{figure*}[t]
\centering
\includegraphics[width=\textwidth]{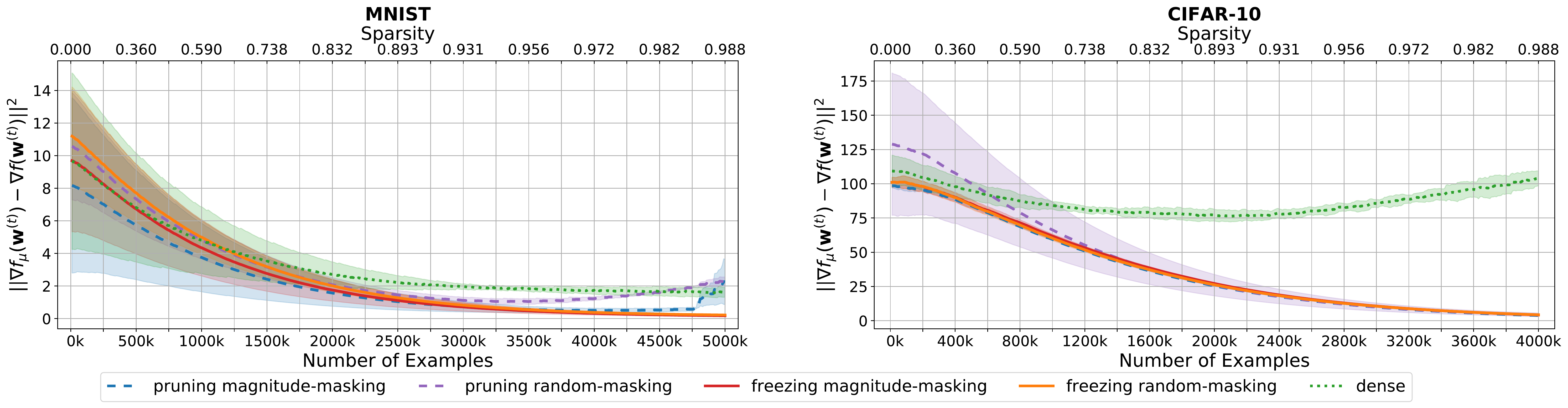}
\caption{Distance of gradient approximation to true gradient on training set, interpolation factor $0.99$.}
\label{fig:gradnorm}
\end{figure*}

\subsection{Masking strategies}
Algorithm \ref{algorithm1} starts with standard SZO optimization with full perturbations in the first iteration. This is done by initializing the mask to a $n$-dimensional vector of coordinates with value $1$ (line \ref{line:init_mask}). Guided by a schedule that applies sparsification at every $s$-th iteration, the  function $\mathrm{get\_mask(\cdot)}$ is applied (line \ref{line:get_mask}). We implemented a first strategy called \emph{magnitude masking}, and another one called \emph{random masking}. In \emph{magnitude masking}, we sort the indices according to their L1-norm magnitude, set a cut-off point, and mask the weights below the threshold. In \emph{random masking}, we sample 50 random mask patterns and select the one that performs best according to accuracy on a heldout set. Following \cite{FrankleETAL:19}, the sparsification interval corresponds to reducing 20\% of the remaining unmasked parameters at every 5 epochs.

\begin{figure*}[t]
\centering
\includegraphics[width=\textwidth]{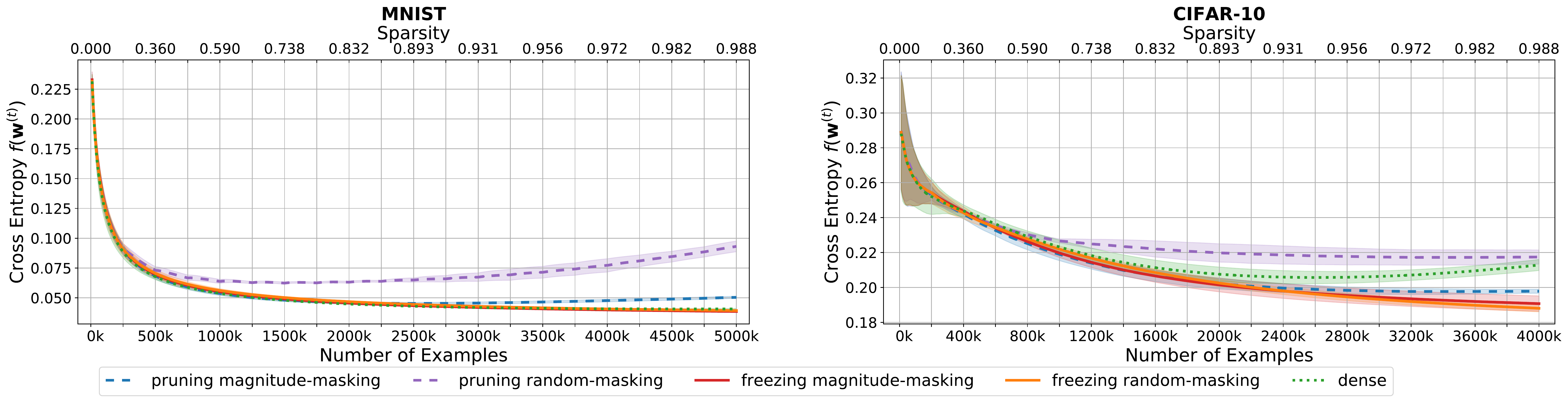}
\caption{Cumulative values for cross-entropy loss of sparse and dense SZO algorithms on the training set.}
\label{fig:train_loss}
\end{figure*}

\begin{figure*}[t]
\centering
\includegraphics[width=\textwidth]{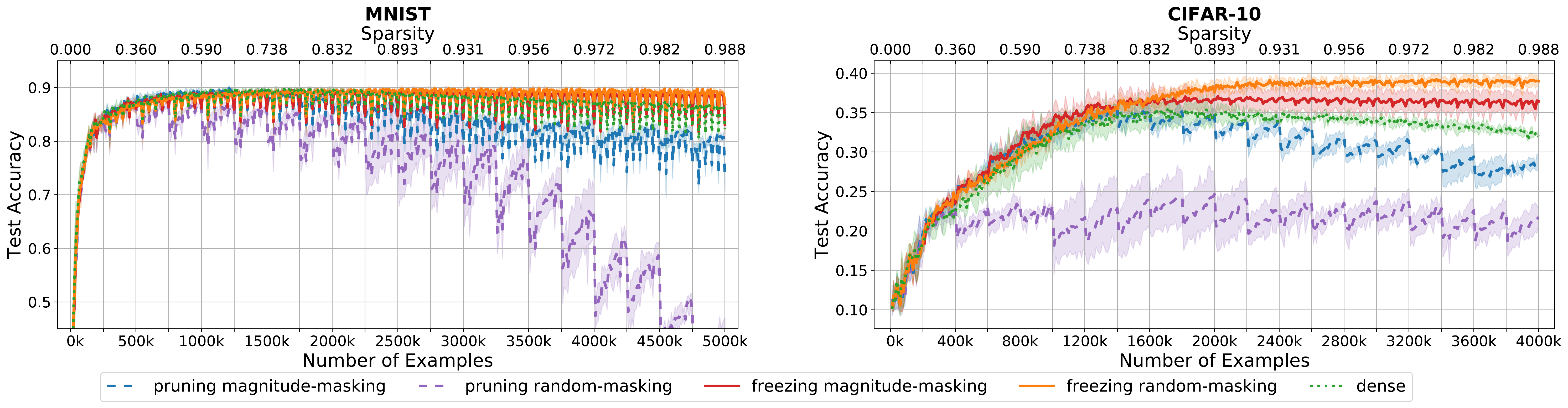}\\
\caption{Multi-class classification accuracy of sparse and dense SZO algorithms on test sets.}
\label{fig:test_acc}
\end{figure*}

\subsection{Sparse perturbations with pruning or freezing}
Algorithm \ref{algorithm1} is defined in a general form that allows \emph{pruning} and \emph{freezing} of masked parameters. In the \emph{pruning} variant, the same sparsification mask that is applied to the Gaussian perturbations (line \ref{line:mask_perturbations}) is also applied to the weight vector itself (line \ref{line:mask_weights}). That is, we keep the weight values from the previous iteration at the index whose mask value is one, and reset the weight values to zero at the index whose mask value is zero. This can be seen as a straightforward sparse-SZO extension of the iterative magnitude pruning method with rewinding to the value of the previous iteration as proposed by \cite{FrankleETAL:19}.  In the \emph{freezing} variant, we apply the sparsification mask only to the Gaussian perturbations (line \ref{line:mask_perturbations}), and inherit all the weight values from the previous iteration. That is, we keep the weight value unchanged at the index whose mask value is zero, and allow the value to be updated at the index whose mask value is one.

\begin{figure*}[t]
\centering
\includegraphics[width=\textwidth]{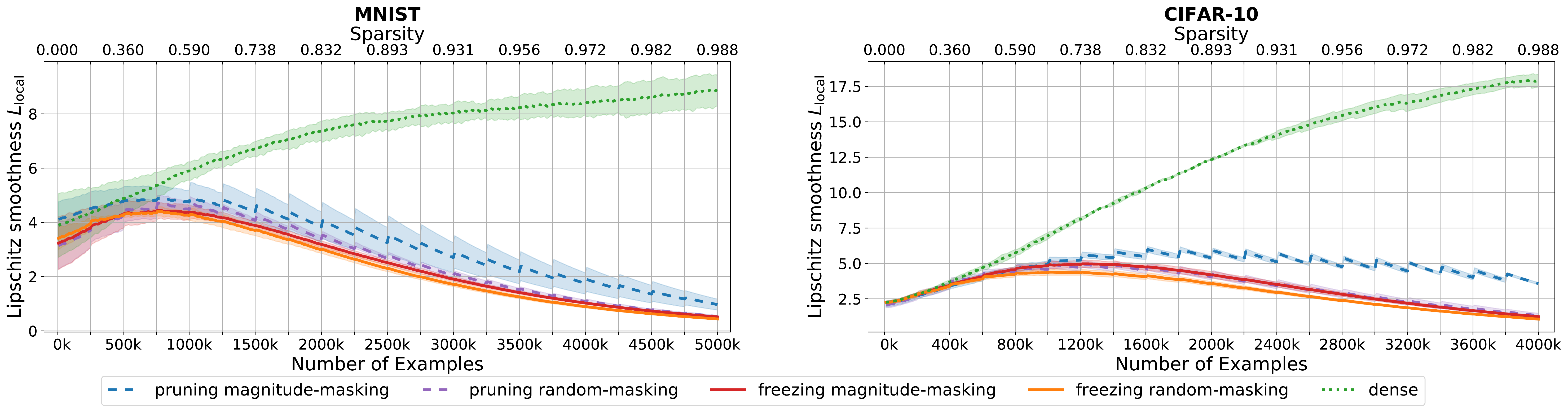}
\caption{Estimated local Lipschitz smoothness $L_{\mathrm{local}}$ of sparse and dense SZO algorithms along search path, interpolation factor $0.99$.} 
\label{fig:lipschitz}
\end{figure*}

\begin{figure*}[t]
\centering
\includegraphics[width=\textwidth]{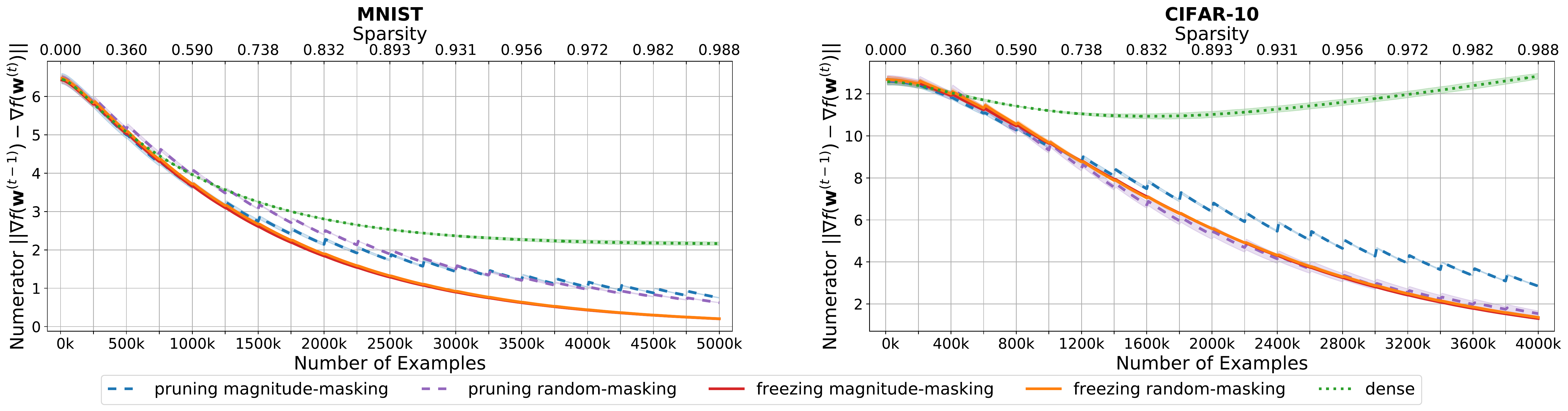}\\
\includegraphics[width=\textwidth]{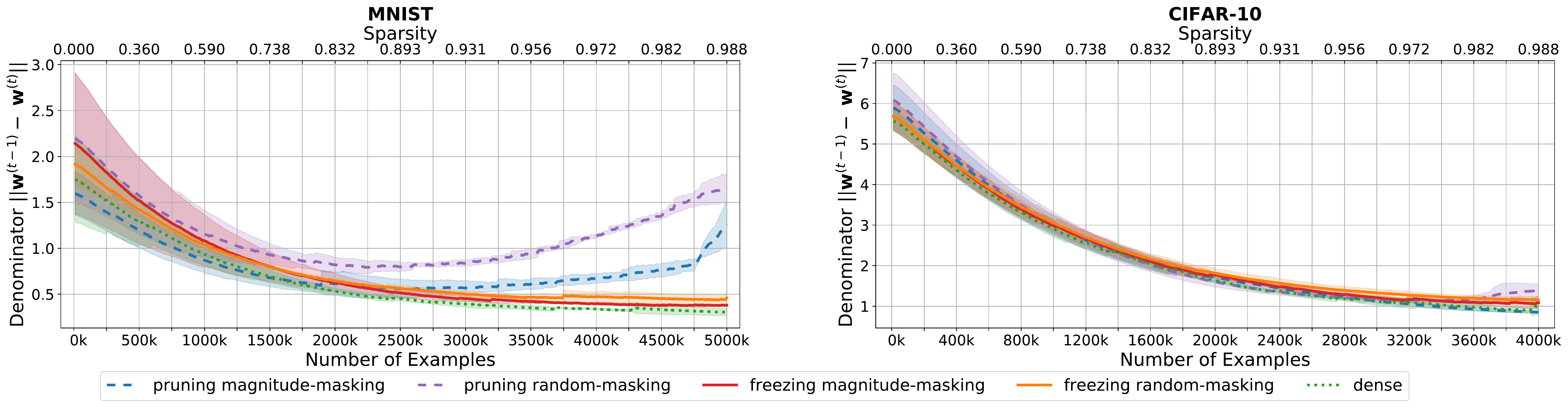}\\
\caption{Denominator and numerator of estimated local Lipschitz smoothness $L_{\mathrm{local}}$ along search path, interpolation factor $0.99$.}
\label{fig:lipschitz-denominator-numerator}
\end{figure*}

\section{Experiments}

We purposely chose experimental tasks where the true gradient can be calculated exactly in order to experimentally verify the improved gradient approximation by sparse perturbations in SZO. While the convergence rate in SZO is always is suboptimal compared to first-order optimization due to the used gradient approximation (see \cite{DuchiETAL:15} or \cite{Shamir:17}), our proofs show that sparse perturbations improve the distance of the approximate gradient to the true gradient (Lemma \ref{lemma3}) and thus the convergence rate (Theorem \ref{theorem1}). The experimental results given in the following confirm our theoretical analysis.

\subsection{Task and datasets}
We apply our sparse SZO algorithm to standard image classification benchmarks on the MNIST \cite{LeCun:98} and CIFAR-10 \cite{Krizhevsky:09} datasets. Since the SZO optimizers are less common in an off-the-shelf API package bundled in a framework, we implemented our custom SZO optimizers utilizing \texttt{pytorch}, and the data was downloaded via torchvision package v0.4.2. We followed the pre-defined train-test split, and partitioned 20\% of the train set for validation. 
Model performance is measured by successive evaluation of the gradient norm on the training set, multi-class classification accuracy on the test set, cumulative cross-entropy loss on the training set, and the local function smoothness, estimated by the local Lipschitz constant computed as
\begin{align}
L_{\mathrm{local}} := \frac{\norm{\nabla f(\w^{(t-1)}) - \nabla f(\w^{(t)})}}{\norm{\w^{(t-1)} - \w^{(t)}}}.
\label{local_lipschitz}
\end{align}

\subsection{Architectures and experimental settings}
For the MNIST experiments, we constructed a fully-connected feed-forward neural network models with 3 layers including batch normalization \cite{Ioffe:15} after each layer. For the CIFAR-10 experiments, we built a convolutional neural network with 2 convolutional layers with kernel size 3 and stride 1. After taking max pooling of these two convolutional layers, we use 2 extra fully-connected linear layers on top of it. We employed the xavier-normal weights initialization method \cite{Glorot:10} for all setups. Each model is trained by minimizing a standard cross-entropy loss for 100 epochs with batch size 64. We used a constant learning rate of 0.2 for MNIST and 0.1 for CIFAR, a constant smoothing parameter of 0.05, and two-sided perturbations in the SZO optimization for both datasets. We sampled 10 perturbations per update from a zero-centered Gaussian distribution with variance 1.0, computed the gradients for each, then the average was taken before update. 
Applying the 20\% reduction schedule 20 times throughout the training, the active model parameters were reduced from 266,610 to 3,844 in MNIST experiments, from 4,301,642 to 61,996 in CIFAR experiments. We computed sparsity as $1-\frac{\bar{n}^{(t)}}{n^{(t)}}$ every time we update the mask pattern. The sparsity value along iteration is given in the upper x-axis of each plot. The lower $x$-axis shows the number of examples visited during training in thousands. If indicated, plotted curves are smoothed by linear interpolation with the previous values. All plots show mean and standard deviation of results for 3 training runs under different random seeds. A summary of data statistics and system settings can be found in Table \ref{table:hyperparameters}.

\subsection{Experimental results}
Our experimental task was purposely chosen as an application that allows optimization with standard first-order (SFO) gradient-based techniques for comparison with SZO methods. 
Figure \ref{fig:gradnorm} shows that the distance of the approximate gradient to the true gradient converges to zero at the fastest rate for freezing methods, proving empirical support for \eqref{d1} in Lemma \ref{lemma3}. All four sparse SZO variants approach the true gradients considerably faster in most training steps on both training datasets than dense SZO (i.e.,  the full dimension of parameters is perturbed in each training step). We see that dense SZO is diverging from the true gradient on the CIFAR dataset where more than 4 million parameters are trained.

Figure \ref{fig:train_loss} gives cumulative values of the cross-entropy loss function on the respective training sets. We see that convergence in cross-entropy loss is similar for sparse and dense SZO variants on MNIST and advantageous for sparse SZO on CIFAR, with both weight pruning methods showing a slight divergence in cross-entropy loss.

Figure \ref{fig:test_acc} shows the accuracy of multi-class classification on MNIST and CIFAR-10 for the sparse and dense SZO variants. We see the fastest convergence in test accuracy for freezing methods, compared to pruning and dense SZO. Interestingly, we see that the test accuracy decreases after a while for all pruning variants, while it continues to increase for the freezing variants and the dense SZO algorithm.

\begin{figure*}[t]
\centering
\includegraphics[width=\textwidth]{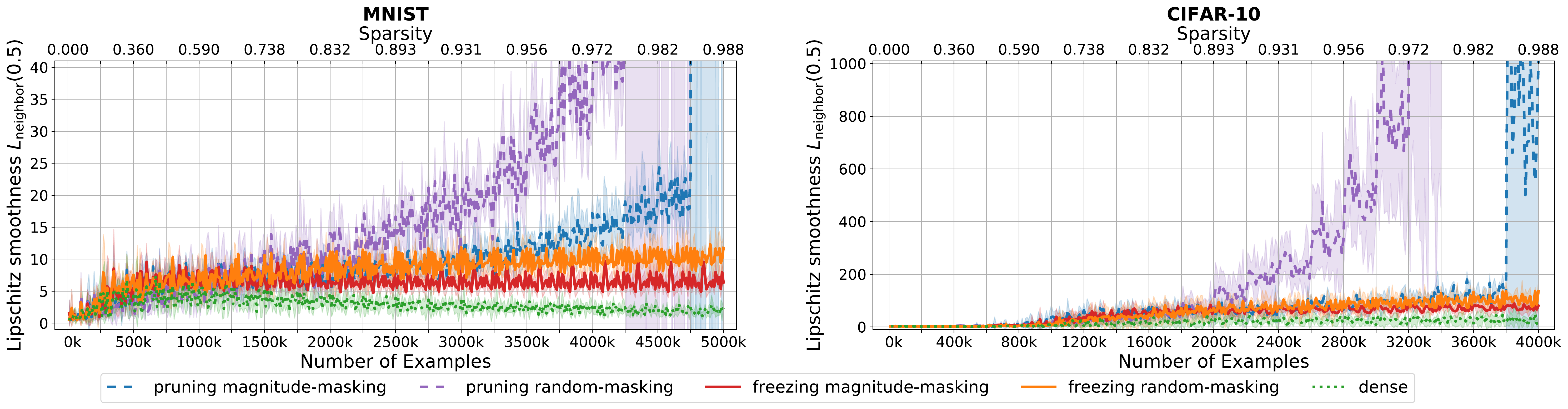}
\caption{Estimated local Lipschitz smoothness $L_{\mathrm{neighbor}}$ on test set in neighborhood of $w^{(t)}$, interpolation factor $0.8$.}
\label{fig:lipschitz-neighborhood}
\end{figure*}

One possibility to analyze this effect is to inspect the estimate of the local Lipschitz smoothness $L_{\mathrm{local}}$ defined in~\eqref{local_lipschitz} along the current search path on the training set. Figure~\ref{fig:lipschitz} shows that $L_{\mathrm{local}}$ decreases considerably along the search trajectory for sparse SZO optimization, with a faster decrease for freezing than for pruning methods. However, it increases considerably for the dense variant, leading the optimization procedure into less smooth areas of the search space. In order to verify that this is not a trivial effect due to pruned weight vectors, we consider numerator and denominator of $L_{\mathrm{local}}$ separately in Figure \ref{fig:lipschitz-denominator-numerator}. We see that the denominator and numerator are decreasing at different speeds for the sparse variants, with relatively larger differences between gradients causing an increase in the Lipschitz quotient. The numerator $\norm{\nabla f(w^{(t-1)}) - \nabla f(w^{(t)})}$ is successively reduced in the sparse methods, with a faster reduction for freezing than for pruning. The reduction indicates that sparse SZO methods find a path converging to the optimal point, while the dense variant seems not to be able to find such a path, but instead stays at the same distance or even increases the distance from the previous step.

Finally, we evaluate an estimate of the local Lipschitz smoothness in a neighborhood $\vb \sim \mathrm{Uniform}(-0.5,0.5)$ around the current parameter values $\w^{(t)}$ on the test set. $L_{\mathrm{neighbor}}$ is estimated by taking  the maximum of 10 samples per batch as 
\begin{align}
L_{\mathrm{neighbor}} := \underset{\vb}{\max} \frac{\norm{\nabla f(\w^{(t)}) - \nabla f(\w^{(t)}+\vb)}}{\norm{\vb}}.
\label{local_lipschitz_neighbor}
\end{align}
Figure \ref{fig:lipschitz-neighborhood} shows that $L_{\mathrm{neighbor}}$ is staying roughly constant on the test set in the neighborhood of the parameter estimates obtained during training for freezing and dense SZO. However, it is increasing considerably for both variants of pruning, indicating that pruning methods converge to less smooth regions in parameter space.

\section{Conclusion}

SZO methods are flexible and simple tools for provably convergent optimization of black-box functions only from function evaluations at random points. A theoretical analysis of these methods can be given under mild assumptions of Lipschitz-smoothness of the function to be optimized. However, even in the best case, such gradient-free methods suffer a factor of $\sqrt{n}$ in iteration complexity, depending on the dimensionality $n$ of the evaluated function, compared to methods that employ gradient information. This makes SZO techniques always second-best if gradient information is available, especially in high dimensional optimization. However, SZO techniques may be the only option in black-box optimization scenarios where it is desirable as well to improve convergence speed. Our paper showed that the dimensionality factor can be reduced to the expected dimensionality of random perturbations during training, independent of the dimensionality reduction schedule employed. This allows considerable speedups in convergence of SZO optimization in high dimensions. 
We presented experiments with masking schedules based on L1-magnitude and random masking, both confirming our theoretical result of an improved approximation of the true gradient by a sparse SZO gradient. This result is accompanied by an experimental finding of improved convergence in training loss and test accuracy. 
A further experiment compared sparse SZO where masked parameters are frozen at their current values to a variant of sparse SZO where masked parameters are pruned to zero values. This technique can be seen as a sparse SZO variant of the well-known technique of iterative magnitude pruning \cite{FrankleCarbin:19,HanETAL:15}. Our results show that pruning techniques perform worse in the world of SZO optimization than fine-tuning unmasked parameters while freezing masked parameters. By further inspection of the local Lipschitz smoothness along the search path of different algorithms we find that pruning in the SZO world can lead the optimization procedure in a locally less smooth search area, while freezing seems to lead to a smoother search path even compared to dense SZO optimization.

\section*{Acknowledgements}
 We would like to thank Michael Hagmann for proofreading the theoretical analysis, and the anonymous reviewers for their useful comments. The research reported in this paper was supported in part by the German research foundation (DFG) under grant RI-2221/4-1.

%
%
%
\bibliographystyle{splncs04}
\bibliography{paper}

\begin{thebibliography}{10}
\providecommand{\url}[1]{\texttt{#1}}
\providecommand{\urlprefix}{URL }
\providecommand{\doi}[1]{https://doi.org/#1}

\bibitem{AgarwalETAL:10}
Agarwal, A., Dekel, O., Xiao, L.: Optimal algorithms for online convex
  optimization with multi-point bandit feedback. In: {COLT} (2010)

\bibitem{BalasubramanianGhadimi:18}
Balasubramanian, K., Ghadimi, S.: Zeroth-order nonconvex stochastic
  optimization: Handling constraints, high-dimensionality and saddle-points.
  CoRR  \textbf{abs/1809.06474} (2018)

\bibitem{ChenETAL:17}
Chen, P.Y., Zhang, H., Sharma, Y., Yi, J., Hsieh, C.J.: Zoo: Zeroth order
  optimization based black-box attacks to deep neural networks without training
  substitute models. In: {AISec} (2017)

\bibitem{ChengETAL:19}
Cheng, S., Dong, Y., Pang, T., Su, H., Zhu, J.: Improving black-box adversarial
  attacks with a transfer-based prior. In: NeurIPS (2019)

\bibitem{DuchiETAL:15}
Duchi, J.C., Jordan, M.I., Wainwright, M.J., Wibisono, A.: Optimal rates for
  zero-order convex optimization: The power of two function evaluations. {IEEE}
  Trans. on Information Theory  \textbf{61}(5),  2788--2806 (2015)

\bibitem{EbrahimiETAL:17}
Ebrahimi, S., Rohrbach, A., Darrell, T.: Gradient-free policy architecture
  search and adaptation. In: Proceedings of the Conference on Robot Learning
  {(CoRL)}. Mountain View, {CA, USA} (2017)

\bibitem{FlaxmanETAL:05}
Flaxman, A.D., Kalai, A.T., McMahan, H.B.: Online convex optimization in the
  bandit setting: gradient descent without a gradient. In: {SODA} (2005)

\bibitem{FrankleCarbin:19}
Frankle, J., Carbin, M.: The lottery ticket hypothesis: Finding sparse,
  trainable neural networks. In: ICLR (2019)

\bibitem{FrankleETAL:19}
Frankle, J., Dziugaite, G.K., Roy, D.M., Carbin, M.: Stabilizing the lottery
  ticket hypothesis. CoRR  \textbf{abs/1903.01611} (2019)

\bibitem{Fu:06}
Fu, M.C.: Gradient estimation. In: Henderson, S., Nelson, B. (eds.) Handbook in
  Operations Research and Management Science, vol.~13, pp. 575--616. Elsevier
  (2006)

\bibitem{GhadimiLan:12}
Ghadimi, S., Lan, G.: Stochastic first- and zeroth-order methods for nonconvex
  stochastic programming. {SIAM} Journal on Optimization  \textbf{4}(23),
  2342--2368 (2012)

\bibitem{Glorot:10}
Glorot, X., Bengio, Y.: Understanding the difficulty of training deep
  feedforward neural networks. In: PMLR. vol.~9, pp. 249--256 (2010)

\bibitem{HanETAL:15}
Han, S., Pool, J., Tran, J., Dally, W.: Learning both weights and connections
  for efficient neural network. In: NIPS (2015)

\bibitem{Ioffe:15}
Ioffe, S., Szegedy, C.: Batch normalization: accelerating deep network training
  by reducing internal covariate shift. In: ICML. pp. 448--456 (2015)

\bibitem{KeskarETAL:17}
Keskar, N.S., Mudigere, D., Nocedal, J., Smelyanskiy, M., Tang, P.T.P.: On
  large-batch training for deep learning: Generalization gap and sharp minima.
  In: ICLR (2017)

\bibitem{KieferWolfowitz:52}
Kiefer, J., Wolfowitz, J.: Stochastic estimation of the maximum of a regression
  function. Annals of Mathematical Statistics  \textbf{23}(3),  462--466 (1952)

\bibitem{Krizhevsky:09}
Krizhevsky, A., Hinton, G.: Learning multiple layers of features from tiny
  images. Master's thesis, Department of Computer Science, University of Tront
  (2009)

\bibitem{KushnerYin:03}
Kushner, H.J., Yin, G.G.: Stochastic Approximation and Recursive Algorithms and
  Applications. Springer, second edn. (2003)

\bibitem{LeCun:98}
LeCun, Y., Bottou, L., Bengio, Y., Haffner, P.: Gradient-based learning applied
  to document recognition. Proc. of the IEEE  \textbf{86}(11),  2278--2324
  (1998)

\bibitem{LiuETAL:18}
Liu, S., Kailkhura, B., Chen, P.Y., Ting, P., Chang, S., Amini, L.:
  Zeroth-order stochastic variance reduction for nonconvex optimization. In:
  Advances in Neural Information Processing Systems 31. Montreal, Canada (2018)

\bibitem{ManiaETAL:18}
Mania, H., Guy, A., Recht, B.: Simple random search provides a competitive
  approach to reinforcement learning. In: NIPS (2018)

\bibitem{NesterovSpokoiny:15}
Nesterov, Y., Spokoiny, V.: Random gradient-free minimization of convex
  functions. Foundations of Computational Mathematics  \textbf{17},  527--566
  (2015)

\bibitem{PlappertETAL:18}
Plappert, M., Houthooft, R., Dhariwal, P., Sidor, S., Chen, R.Y., Chen, X.,
  Asfour, T., Abbeel, P., Andrychowicz, M.: Parameter space noise for
  exploration. In: {ICLR} (2018)

\bibitem{SalimansETAL:17}
Salimans, T., Ho, J., Chen, X., Sutskever, I.: Evolution strategies as a
  scalable alternative to reinforcement learning. CoRR  \textbf{abs/1703.03864}
  (2017)

\bibitem{SehnkeETAL:10}
Sehnke, F., Osendorfer, C., R{\"u}ckstie{\ss}, T., Graves, A., Peters, J.,
  Schmidhuber, J.: Parameter-exploring policy gradients. Neural Networks
  \textbf{23}(4),  551--559 (2010)

\bibitem{Shamir:17}
Shamir, O.: An optimal algorithm for bandit and zero-order convex optimization
  with two-point feedback. JMLR  \textbf{18},  1--11 (2017)

\bibitem{SokolovETAL:18}
Sokolov, A., Hitschler, J., Riezler, S.: Sparse stochastic zeroth-order
  optimization with an application to bandit structured prediction. CoRR
  \textbf{abs/1806.04458} (2018)

\bibitem{Spall:92}
Spall, J.C.: Multivariate stochastic approximation using a simultaneous
  perturbation gradient approximation. {IEEE} Transactions on Automatic Control
   \textbf{37}(3),  332--341 (1992)

\bibitem{Spall:03}
Spall, J.C.: Introduction to Stochastic Search and Optimization: Estimation,
  Simulation, and Control. Wiley (2003)

\bibitem{WangETAL:18}
Wang, Y., Du, S., Balakrishnan, S., Singh, A.: Stochastic zeroth-order
  optimization in high dimensions. In: AISTATS (2018)

\bibitem{YueJoachims:09}
Yue, Y., Joachims, T.: Interactively optimizing information retrieval systems
  as a dueling bandits problem. In: {ICML} (2009)

\end{thebibliography}
%

\newpage

\section*{Appendix}
\renewcommand{\thesubsection}{\Alph{subsection}}

\subsection{Notation}

In order to adapt the theoretical analysis of \cite{NesterovSpokoiny:15} to the case of sparse perturbations, the probability density function $\p_{\ub}(\ub)$ and expectations $\E_{\ub}[f(\ub)]$ of random variables with respect to it have to be carefully defined.
Given a sparsification mask $\m \in \{0,1\}^{n}$, let $I_{\mathrm{masked}} := \{i | m_i = 0\}$ be a set of indices whose mask value is zero. Likewise, let $I_{\mathrm{unmasked}} := \{i | m_i = 1\}$ be a set of indices whose mask value is one. The probability density function of random perturbations $\ug \sim \N(\bm{0}, \bm{I})$ can be expressed as a product of univariate probability density functions $\p(u_i)$, where $u_i \sim \N(0,1)$. Then the probability density function $\p_{\ub}(\ub)$ can be defined as a conditional probability density function $\p_{\ug|u_i = 0, i\in I_{\mathrm{masked}}}(\ug)$, conditioned on zero-valued entries $u_i = 0$ corresponding to masked indices $i\in I_{\mathrm{masked}}$. Since the number of nonzero parameters is determined by the sparsification mask as $|I_{\mathrm{unmasked}}| = \bar{n}$, we can reduce the product over the full dimensionality $i \in \{1, \cdots n\}$ to the product over unmasked entries $i \in I_{\mathrm{unmasked}}$. The conditional probability density function $\p_{\ug|u_i = 0, i\in I_{\mathrm{masked}}}(\ug)$ is then defined as
\begin{align}
\p_{\ug|u_i = 0, i\in I_{\mathrm{masked}}}(\ug) & = \prod_{i\in I_{\mathrm{unmasked}}} \p(u_i) \\
& = \left(\frac{1}{\sqrt{2\pi}}\right)^{\bar{n}} \prod_{i\in I_{\mathrm{unmasked}}} e^{-\frac{1}{2}u_i^2}.
\end{align}
The expectation with respect to $\ub$ is defined as conditional expectation using aforementioned conditional probability density function. Let $f: \R^n \rightarrow \R$ be an arbitrary function, let $\int_{\ug\upharpoonright} \cdot d\ug$ denote the integral $\int_{\ug|u_i=0, i\in I_{\mathrm{masked}}} \cdot d\ug$, and let $\p_{\ug\upharpoonright}$ denote the conditional probability density function $\p_{\ug|u_i = 0, i\in I_{\mathrm{masked}}}$, and let $\kappa := \sqrt{(2\pi)^{\bar{n}}}$. Then we have
\begin{align}
\E_{\ub}[f(\ub)] &:= \E_{\ug|u_i=0, i\in I_{\mathrm{masked}}}[f(\ug)] \\
&= \int_{\ug\upharpoonright}  ~f(\ug)\cdot \p_{\ug\upharpoonright}(\ug) d\ug\\
&= \frac{1}{\kappa} \int_{\ug\upharpoonright} ~f(\ug)\cdot \pdf d\ug.
\end{align}

In the following, we use the notation $\int_{\ug\upharpoonright} \cdot d\ug$ for $\int_{\ug|u_i=0, i\in I_{\mathrm{masked}}} \cdot d\ug$, and $\p_{\ug\upharpoonright}$ for $\p_{\ug|u_i = 0, i\in I_{\mathrm{masked}}}$.

\subsection{Proof for Lemma \ref{lemma1}}
\begin{proof}
Define by $\ut \in \R^{\bar{n}}$ a reduced vector that removes all zero entries from $\ub\in \R^n$. Taking an expectation over $\ub \in \R^n$ results in the same value as taking an expectation over $\ut \in \R^n$. This can be seen as follows: 
\begin{align*}
\E_{\ub}\left[\norm{\ub}^p\right] &= \frac{1}{\kappa}\int_{\ug|u_i=0, i\in I_{\mathrm{masked}}}\norm{\ug}^p e^{-\frac{1}{2}\ug^\top \Sigma\ug} d\ug\\
&= \frac{1}{\kappa}\int_{\ug|u_i=0, i\in I_{\mathrm{masked}}} \left(\sum_{j=1}^{n} \lvert u_j \rvert^p \right) e^{-\frac{1}{2}\ug^\top \Sigma\ug} d\ug\\
&= \frac{1}{\kappa}\sum_{i\in I_{\mathrm{unmasked}}}\left(\int_{u_i} \lvert u_i \rvert^p  e^{-\frac{1}{2}u_i^2} d u_i \right)\\
&= \frac{1}{\kappa}\int_{\ut} \left(\sum_{i=1}^{\bar{n}} \lvert u_i \rvert^p \right) e^{-\frac{1}{2}\ut^\top \tilde{\Sigma}\ut} d\ut\\
&= \frac{1}{\kappa} \int_{\ut}\norm{\ut}^p e^{-\frac{1}{2}\ut^\top \tilde{\Sigma}\ut} d\ut\\
&= \E_{\ut}\left[\norm{\ut}^p\right]
\end{align*}
where $\tilde{\Sigma}$ is a $\bar{n} \times \bar{n}$ unit covariance matrix with $\tilde{\Sigma}=\bm{I}$. 

Let $p \geq 2$ and $\tau \in (0,1)$. Recalling that $\ub \in \R^{n}$, $\ut \in \R^{\bar{n}}$, and $\tilde{\Sigma}\in \R^{\bar{n}\times \bar{n}}$, we have 
\begin{align}
\E_{\ub}\left[\norm{\ub}^p\right] &= \E_{\ut}\left[\norm{\ut}^p\right] \nonumber\\
&=\frac{1}{\sqrt{(2\pi)^{\bar{n}}}} \int_{\ut}\norm{\ut}^p e^{-\frac{\tau + (1-\tau)}{2} \ut^\top \ut} d\ut \nonumber\\
&= \frac{1}{\sqrt{(2\pi)^{\bar{n}}}} \int_{\ut}\norm{\ut}^p e^{-\frac{\tau}{2} \ut^\top \ut}e^{-\frac{1-\tau}{2} \ut^\top \ut} d\ut \nonumber\\
&\leq \frac{1}{\sqrt{(2\pi)^{\bar{n}}}} \int_{\ut} \left(\frac{p}{\tau e}\right)^{\frac{p}{2}} e^{-\frac{1-\tau}{2} \ut^\top \ut} d\ut \label{e1}\\
&= \left(\frac{p}{\tau e}\right)^{\frac{p}{2}} \frac{1}{\sqrt{(2\pi)^{\bar{n}}}}\int_{\ut} e^{-\frac{1}{2} \ut^\top\left(\frac{\tilde{\Sigma}}{1-\tau}\right)^{-1} \ut} d\ut \nonumber\\
&= \left(\frac{p}{\tau e}\right)^{\frac{p}{2}} \frac{1}{\sqrt{(2\pi)^{\bar{n}}}} \sqrt{(2\pi)^{\bar{n}} \cdot \det\left(\frac{\bar{\Sigma}}{1-\tau}\right)}\label{e5}\\
&= \left(\frac{p}{\tau e}\right)^{\frac{p}{2}} (1-\tau)^{-\frac{\bar n}{2}}\label{e2}\\
&\leq \left(\bar{n} + p\right)^{\frac{p}{2}} \label{e3}.
\end{align}
The first inequality \eqref{e1} follows from $t^p e^{-\frac{\tau}{2}t^2} \leq \left(\frac{p}{\tau e}\right)^\frac{p}{2}$ for $ t> 0$. The second inequality \eqref{e3} follows by minimizing $\left(\frac{p}{\tau e}\right)^{\frac{p}{2}} (1-\tau)^{-\frac{n}{2}}$ in $\tau \in (0,1)$. 
Full proofs of these inequalities are given in subsection \ref{misc}.
\end{proof}

\subsection{Proof for Lemma \ref{lemma2}}\label{proof:lemma2}
\begin{proof}

\begin{align}
\nabla_\w f_{\mu}(\w) &= \nabla_\w \E_{\ub}\left[f(\w + \mu\ub)\right] \nonumber\\
&= \nabla_\w \int_{\ug\upharpoonright} \p_{\ug\upharpoonright}(\ug)~ f(\w + \mu\ug) ~d\ug \nonumber\\
&= \int_{\ug\upharpoonright} \nabla_\w ~\p_{\ug\upharpoonright}(\ug)~ f(\w + \mu\ug) ~d\ug \label{interchange}\\
&= \frac{1}{\kappa}\int_{\ug\upharpoonright} \nabla_\w ~ \pdf f(\w + \mu\ug) ~d\ug \nonumber\\
&= \frac{1}{\kappa}\int_{\y\upharpoonright} \nabla_\w ~ e^{-\frac{1}{2}\lVert\frac{\y - \w}{\mu}\rVert^2} f(\y) \frac{1}{\mu^n}~ d\y  \label{substitute_y}\\
&= \frac{1}{\kappa}\int_{\y\upharpoonright} \frac{\y - \w}{\mu^2 } e^{-\frac{1}{2\mu^2}\lVert\y - \w\rVert^2} f(\y) \frac{1}{\mu^n}~ d\y \nonumber\\
&= \frac{1}{\kappa}\int_{\ug\upharpoonright} \frac{\ug}{\mu} e^{-\frac{1}{2}\lVert\ug\rVert^2} f(\w + \mu\ug) ~d\ug \label{substitute_u}\\
&= \E_{\ub}\left[\frac{f(\w + \mu\ub)}{\mu} \ub \right] \nonumber
\end{align}

Since $f\in C^{1,1}$, we can interchange $\nabla$ and the integral in \eqref{interchange}. Furthermore, we substituted $\y = \w + \mu\ug$ in \eqref{substitute_y} and put it back in \eqref{substitute_u}.

Now, we show \eqref{eq:2-point} $\Leftrightarrow$ \eqref{eq:1-point}.
\begin{align}
&\E_{\ub}\left[\frac{f(\w + \mu\ub) - f(\w)}{\mu} \ub \right] \\
&= \frac{1}{\kappa}\int_{\ug\upharpoonright} \frac{f(\w + \mu\ug) - f(\w)}{\mu} \ug ~  \pdf d\ug \nonumber\\
&= \frac{1}{\kappa}\int_{\ug\upharpoonright} \frac{f(\w + \mu\ug)}{\mu} \ug ~  \pdf d\ug  - \frac{f(\w)}{\mu\kappa} \underset{=0}{\underbrace{ \int_{\ug\upharpoonright} \ug ~  \pdf d\ug}} \nonumber\\
&= \E_{\ub}\left[\frac{f(\w + \mu\ub)}{\mu} \ub \right] \label{eq:two-point1}
\end{align}

Lastly, we show \eqref{eq:2-sided} $\Leftrightarrow$ \eqref{eq:1-point}.
The expectation doesn't change even if we shift $\w$ by $\mu\ub$, hence 
\begin{align}
\E_{\ub}\left[\frac{f(\w) - f(\w - \mu\ub)}{\mu} \ub \right] &= \E_{\ub}\left[\frac{f(\w + \mu\ub) - f(\w)}{\mu} \ub \right] \\
&\overset{\eqref{eq:two-point1}}{=} \E_{\ub}\left[\frac{f(\w + \mu\ub)}{\mu} \ub \right].\label{eq:two-point2}
\end{align}
\begin{align*}
&\E_{\ub}\left[\frac{f(\w + \mu\ub) - f(\w - \mu\ub)}{2\mu} \ub \right]\\
&= \frac{1}{2} \left(\E_{\ub}\left[\frac{f(\w + \mu\ub)}{\mu} \ub \underset{=0}{\underbrace{- \frac{f(\w)}{\mu} \ub + \frac{f(\w)}{\mu}}} \ub - \frac{f(\w - \mu\ub)}{\mu} \ub \right]\right)\\
&= \frac{1}{2} \left(\E_{\ub}\left[\frac{f(\w + \mu\ub) - f(\w)}{\mu} \ub \right]  + \E_{\ub}\left[\frac{f(\w) - f(\w - \mu\ub)}{\mu} \ub \right]\right)\\
&\overset{\eqref{eq:two-point2}}{=} \frac{1}{2} \left(\E_{\ub}\left[\frac{f(\w + \mu\ub)}{\mu} \ub \right]  + \E_{\ub}\left[\frac{f(\w + \mu\ub)}{\mu} \ub \right]\right)\\
&= \E_{\ub}\left[\frac{f(\w + \mu\ub)}{\mu} \ub \right]
\end{align*}
\end{proof}

\subsection{Proof for Lemma \ref{lemma3}}
\begin{proof}
\begin{align}
&\norm{\nabla f_\mu(\w) - \nabla f(\w)} \\
&= \left\lVert\frac{1}{\kappa} \int_{\ug\upharpoonright} \left(\frac{f(\w+\mu\ug) - f(\w)}{\mu} - \langle \nabla f(\w),\ug \rangle \right) ~ \ug ~ e^{-\frac{1}{2}\lVert\ug\rVert^2} d\ug \right\rVert\label{d1}\\
&\leq \frac{1}{\kappa\mu}\int_{\ug\upharpoonright} \lvert f(\w+\mu\ug) - f(\w) - \mu\langle \nabla f(\w),\ug \rangle \rvert \norm{\ug} ~ e^{-\frac{1}{2}\lVert\ug\rVert^2} d\ug\label{d2}\\
&\overset{\eqref{def:lipschitz}}{\leq} \frac{\mu L(f)}{2\kappa} \int_{\ug\upharpoonright} \norm{\ug}^3 ~ e^{-\frac{1}{2}\lVert\ug\rVert^2} d\ug \nonumber\\
&= \frac{\mu L(f)}{2} \E_{\ub}\left[\norm{\ub}^3\right] \nonumber\\
&\overset{\eqref{eq:lemma1}}{\leq}\frac{\mu L(f)}{2} (\bar{n} + 3)^{\sfrac{3}{2}} \label{d3}
\end{align}

In the first equality \eqref{d1}, we used $\nabla f(\w) \overset{\eqref{eq:misc3}}{=}\E_{\ub}\left[ \langle \nabla f(\w),\ub \rangle \ub  \right] $\\
$= \frac{1}{\kappa}\int_{\ug\upharpoonright} \langle \nabla f(\w),\ug \rangle \ug ~ e^{-\frac{1}{2}\lVert\ug\rVert^2} d\ug$ (see Lemma \ref{misc3}). The first inequality \eqref{d2} follows because $\left\lVert\int f(x)dx \right\rVert \leq \int \left\lVert f(x) \right\rVert dx$ (the triangle inequality for integrals).

Setting $\ab \leftarrow \nabla f(\w)$ and $\bb \leftarrow \nabla f_\mu(\w)$ in Lemma \ref{lemma0}, we get
\begin{align*}
\norm{\nabla f(\w)}^2 &\overset{\eqref{eq:lemma0}}{\leq} 2\norm{\nabla f(\w) - \nabla f_\mu(\w)}^2 + 2\norm{\nabla f_\mu(\w)}^2\\
&\overset{\eqref{d3}}{\leq} \frac{\mu^2 L^2(f)}{2} (\bar{n} + 3)^3 + 2\norm{\nabla f_\mu(\w)}^2\\
&\leq \frac{\mu^2 L^2(f)}{2} (\bar{n} + 4)^3 + 2\norm{\nabla f_\mu(\w)}^2
\end{align*}
\end{proof}

\subsection{Proof for Theorem \ref{theorem1}}
\begin{proof}

\begin{align}
f_\mu(\w) - f(\w) &= \E_{\ub}\left[f(\w + \mu\ub)\right] - f(\w)\nonumber\\
&= \E_{\ub}\left[f(\w + \mu\ub) - f(\w) -\mu\langle \nabla f(\w), \ub\rangle\right]\label{c1}\\
&= \frac{1}{\kappa}\int_{\ug\upharpoonright}\left[f(\w + \mu\ug) - f(\w) -\mu\langle \nabla f(\w), \ug\rangle\right] \pdf d\ug\nonumber\\
&\overset{\eqref{def:lipschitz}}{\leq} \frac{1}{\kappa} \int_{\ug\upharpoonright} \frac{\mu^2 L(f)}{2}\norm{\ug}^2 \pdf d\ug\nonumber\\
&= \frac{\mu^2L(f)}{2}\E_{\ub}\left[\norm{\ub}^2\right] \nonumber\\
&\overset{\eqref{eq:lemma1_p_2}}{\leq} \frac{\mu^2 L(f)}{2}\bar{n}\label{a0}
\end{align}
The second equality \eqref{c1} follows because $\E_{\ub}\left[\ub\right] = \bm{0}$. Moreover, \eqref{a0} doesn't change even if we shift $\w$ by $\mu\ub$, therefore we have
\begin{align}
&\left[(f_\mu(\w) - f(\w)) - (f_\mu(\w + \mu\ug) - f(\w + \mu\ug))\right]^2\\ &\overset{\eqref{eq:misc}}{\leq} 2\left[f_\mu(\w) - f(\w)\right]^2 + 2\left[f_\mu(\w + \mu\ug) - f(\w + \mu\ug)\right]^2 \nonumber\\
&\overset{\eqref{a0}}{\leq} \frac{\mu^4 L^2(f)}{2}\bar{n}^2 + \frac{\mu^4 L^2(f)}{2}\bar{n}^2 \nonumber\\
&= \mu^4 L^2(f)\bar{n}^2\label{a1}
\end{align}

Setting $\ab \leftarrow f_\mu(\w + \mu\ug) - f_\mu(\w)$ and $\bb \leftarrow \mu \langle \nabla f_\mu(\w), \ug \rangle$ in Lemma \ref{lemma0}, we get
\begin{align}
&\left[f_\mu(\w + \mu\ug) - f_\mu(\w)\right]^2\\
&\overset{\eqref{eq:lemma0}}{\leq} 2\left[f_\mu(\w + \mu\ug) - f_\mu(\w) - \mu \langle \nabla f_\mu(\w), \ug \rangle\right]^2 + 2\left[\mu \langle \nabla f_\mu(\w), \ug \rangle\right]^2\nonumber\\
&\overset{\eqref{def:lipschitz}}{\leq} \frac{\mu^4 L^2(f_\mu)}{2}\norm{\ug}^4 + 2\mu^2 \langle \nabla f_\mu(\w), \ug \rangle ^2 \nonumber\\
&\leq \frac{\mu^4 L^2(f)}{2}\norm{\ug}^4 + 2\mu^2 \norm{\nabla f_\mu(\w)}^2\norm{\ug}^2\label{a2}
\end{align}
The last inequality follows because $L(f_\mu)<L(f)$ and the Cauchy-Schwarz inequality $\langle \ab, \bb \rangle^2 \leq \norm{\ab}^2 \norm{\bb}^2$.

Again, setting $\ab \leftarrow f(\w + \mu\w) - f(\w)$ and $\bb \leftarrow f_\mu(\w + \mu\w) - f_\mu(\w)$ in Lemma \ref{lemma0}, we obtain
\begin{align}
&\left[f(\w + \mu\ug) - f(\w)\right]^2\\
&\overset{\eqref{eq:lemma0}}{\leq} 2\left[(f_\mu(\w) - f(\w)) - (f_\mu(\w + \mu\ug) - f(\w + \mu\ug))\right]^2 + 22\left[f_\mu(\w + \mu\ug) - f_\mu(\w)\right]^2 \nonumber\\
&\overset{\eqref{a1},\eqref{a2}}{\leq} 2\mu^4 L^2(f)\bar{n}^2 + \mu^{4} L^2(f)\norm{\ug}^4 + 4\mu^2 \norm{\nabla f_\mu(\w)}^2\norm{\ug}^2 \label{a3}
\end{align}

Now, we evaluate the expectation of $\norm{g_\mu(\w)}^2$ wrt. $\x$ and $\ub$.
\begin{align}
\E_{\ub, \x}\left[\norm{g_\mu(\w)}^2\right] &\overset{\eqref{def:g_mu}}{=} \E_{\ub}\left[\left\lVert\frac{f(\w+\mu\ub) - f(\w)}{\mu}\ub\right\rVert^2\right] \nonumber\\
&= \E_{\ub}\left[\frac{1}{\mu^2}\left[ f(\w+\mu\ub) - f(\w)\right]^2 \cdot \norm{\ub}^2\right] \\
&\overset{(\ref{a3})}{\leq} \E_{\ub}\left[2\mu^{2} L^2(f)\bar{n}^2\norm{\ub}^2 + \mu^{2} L^2(f)\norm{\ub}^6 + 4\norm{\nabla f_\mu(\w)}^2 \norm{\ub}^4\right] \nonumber\\
&\overset{\eqref{eq:lemma1}}{\leq} 2\mu^2 L^2(f)\bar{n}^3 + \mu^2 L^2(f)(\bar{n}+6)^3 + 4(\bar{n}+4)^2\norm{\nabla f_\mu(\w)}^2 \nonumber\\
&\leq 3\mu^2 L^2(f)(\bar{n}+4)^3 + 4(\bar{n}+4)^2\norm{\nabla f_\mu(\w)}^2 \label{eq:lemma4}
\end{align}
The last inequality follows because $2\bar{n}^3 + (\bar{n} + 6)^3 \leq 3(\bar{n}+4)^3$.

From the Lipschitz-smoothness assumption \eqref{def:lipschitz} of $f_\mu$, we have
\begin{align}
f_\mu(\w^{(t+1)}) &\overset{\eqref{def:lipschitz}}{\leq} f_\mu(\w^{(t)}) + \langle \nabla f_\mu(\w^{(t)}), \w^{(t+1)}-\w^{(t)} \rangle + \frac{L(f_\mu)}{2}\norm{\w^{(t)} - \w^{(t+1)}}^2 \nonumber\\
&= f_\mu(\w^{(t)}) - h^{(t)}\langle \nabla f_\mu(\w^{(t)}), g_\mu(\w^{(t)}) \rangle + \frac{ (h^{(t)})^2 L(f_\mu)}{2}\norm{g_\mu(\w^{(t)})}^2 \label{a4}
\end{align}
In \eqref{a4}, we used the update rule $\w^{(t+1)} = \w^{(t)} - h^{(t)} g_\mu(\w^{(t)})$.
Taking the expectation wrt. $\x^{(t)}$ and $\ub^{(t)}$, we have
\begin{align*}
&\E_{\ub, \x}\left[f_\mu(\w^{(t+1)})\right]\\
&\leq \E_{\ub, \x}\left[f_\mu(\w^{(t)}) \right] - h^{(t)}\E_{\ub, \x}\left[\norm{\nabla f_\mu(\w^{(t)})}^2\right] + \frac{ (h^{(t)})^2 L(f_\mu)}{2}\E_{\ub, \x}\left[\norm{g_\mu(\w^{(t)})}^2 \right]\\
&\overset{\eqref{eq:lemma4}}{\leq} \E_{\ub, \x}\left[f_\mu(\w^{(t)}) \right] - h^{(t)}\E_{\ub, \x}\left[\norm{\nabla f_\mu(\w^{(t)})}^2\right] \nonumber\\
&\hspace{2em}+ \frac{(h^{(t)})^2 L(f)}{2}\left(4(\bar{n}^{(t)}+4) \E_{\ub, \x}\left[\norm{\nabla f_\mu(\w^{(t)})}^2\right] + 3\mu^2 L^2(f)(\bar{n}^{(t)}+4)^3 \right)
\end{align*}

Choosing $h^{(t)} = \frac{1}{4(\bar{n}^{(t)}+4)L(f)}$, we have
\begin{align}
\E_{\ub, \x}\left[f_\mu(\w^{(t+1)})\right]
&\leq \E_{\ub, \x}\left[f_\mu(\w^{(t)}) \right] - \frac{1}{8(\bar{n}^{(t)}+4)L(f)}\E_{\ub, \x}\left[\norm{\nabla f_\mu(\w^{(t)})}^2\right]\\
&\hspace{10em}+ \frac{3\mu^2}{32}L(f)(\bar{n}^{(t)}+4) \label{b1}
\end{align}

Recursively applying the inequalities above moving the index from $T+1$ to $0$, rearranging the terms and noting that $f_\mu^\star \leq f_\mu(w^{(t)})$ and $\hat{h} = \frac{1}{4(\hat{n}+4)L}$ where $L(f)\leq L$ for all $f(\w^{(t)})$ results in
\begin{align}
\E_{\bar{\mathcal{U}}, \X}\left[\norm{\nabla f_\mu(\w^{(T)})}^2\right] &\leq 8(\hat{n}+4)L\left[\frac{f_\mu(\w^{(0)})- f_\mu^\star}{T+1} + \frac{3\mu^2}{32}L(\hat{n}+4)\right]\label{b2}
\end{align}
Thus, we obtain
\begin{align}
\E_{\bar{\mathcal{U}}, \X}\left[\norm{\nabla f(\w^{(T)})}^2\right] &\overset{\eqref{eq:lemma3}}{\leq} \frac{\mu^2 L^2}{2} (\hat{n} + 4)^3 + 2\E_{\bar{\mathcal{U}}, \X}\left[\norm{\nabla f_\mu(\w^{(T)})}^2\right] \nonumber\\
&\overset{\eqref{b2}}{\leq} 16(\hat{n}+4)L\frac{f_\mu(\w^{(0)})- f_\mu^\star}{T+1} + \frac{\mu^2 L^2}{2} (\hat{n} + 4)^2\left(\hat{n} + \frac{11}{2}\right) \label{b3}
\end{align}
Now we lower-bound the expected number of iteration. In order to get $\epsilon$-accurate solution $\E_{\bar{\mathcal{U}}, \X}\left[\norm{\nabla f(\w^{(T)})}^2\right] \leq \epsilon^2$, we need to choose $\mu = \Omega\left(\frac{\epsilon}{\hat{n}^{\sfrac{3}{2}}L}\right)  \cdots \refstepcounter{equation}(\theequation)\label{b4}$ so that the second term in the right hand side of \eqref{b3} vanishes wrt. $\epsilon^2$. 
\begin{align*}
16(\hat{n}+4)L\frac{f_\mu(\w^{(0)})- f_\mu^\star}{T+1} + \bigO(\mu^2 L^2 \hat{n}^3) &\overset{\eqref{b4}}{=} 16(\hat{n}+4)L\frac{f_\mu(\w^{(0)})- f_\mu^\star}{T+1} + \bigO(\epsilon^2)\\ &\overset{!}{=} \bigO(\epsilon^2) \\
T &= \bigO\left(\frac{\hat{n}L}{\epsilon^2}\right)
\end{align*}

\end{proof}

\subsection{Miscellaneous}\label{misc}
\begin{lemma}\label{lemma0}
For any $\ab, \bb \in \mathbb{R}^{n}$, we have
\begin{align}
    \norm{\ab}^2 &\leq 2\norm{\ab - \bb}^2 + 2\norm{\bb}^2 \label{eq:lemma0}
\end{align}
\end{lemma}

\begin{proof}
For any $\x, \y \in \mathbb{R}^{n}$, it holds
\begin{align*}
\norm{\x + \y}^2 + \norm{\x - \y}^2 
&= \norm{\x}^2 + 2\langle\x, \y \rangle + \norm{\y}^2 + \norm{\x}^2 - 2\langle\x, \y \rangle + \norm{\y}^2\\
&= 2\norm{\x}^2 + 2\norm{\y}^2 .
\end{align*}
Dropping either $\norm{\x + \y}^2$ or $\norm{\x - \y}^2$ on the left hand side, we get
\begin{align}
\norm{\x \pm \y}^2 &\leq 2\norm{\x}^2 + 2\norm{\y}^2 \label{eq:misc}
\end{align}
Substitute $\x \leftarrow \ab - \bb$ and $\y \leftarrow \bb$ in $\norm{\x + \y}^2 \leq 2\norm{\x}^2 + 2\norm{\y}^2$. Then we get
\begin{align*}
\norm{\ab}^2 &\leq 2\norm{\ab - \bb}^2 + 2\norm{\bb}^2 .
\end{align*}
\end{proof}

\begin{lemma}\label{misc1}
Let be $t>0$, $p\geq 2$, and $\tau \in (0, 1)$. Then we have
\begin{align}
t^p e^{-\frac{\tau}{2}t^2} \leq \left(\frac{p}{\tau e}\right)^\frac{p}{2}. \label{eq:misc1}
\end{align}
\end{lemma}
\begin{proof}
Denote $\psi(t):= t^p e^{-\frac{\tau}{2}t^2}$ for some fixed $\tau\in(0,1)$. Find the point $t$ s.t. $\psi^{'}(t) = 0$.
\begin{align*}
\psi'(t) = pt^{p-1} e^{-\frac{\tau}{2}t^2} + t^p\left(-\tau t  e^{-\frac{\tau}{2}t^2}\right) &\overset{!}{=} 0\\
t&=\sqrt{\sfrac{p}{\tau}} & \text{ since }t>0, \tau\in(0,1)\\
\psi(\sqrt{\sfrac{p}{\tau}})&= \left(\sfrac{p}{\tau}\right)^\frac{p}{2} e^{-\sfrac{p}{2}}\\
&= \left(\frac{p}{\tau e}\right)^\frac{p}{2}
\end{align*}
Now we conduct the second derivative test.
\begin{align*}
\psi''(t) &= \left(\tau^2 t^2 e^{-\frac{\tau}{2} t^2} - \tau e^{-\frac{\tau}{2} t^2}\right) t^p + (p - 1) p e^{-\frac{\tau}{2} t^2} t^{p - 2} - 2 \tau p e^{-\frac{\tau}{2} t^2} t^p\\
&= e^{-\frac{\tau}{2} t^2} \cdot t^{p-2} \left(\tau^2 t^4 -2\tau p t^2 -\tau t^2 + p^2 -p\right)
\end{align*}
\begin{align*}
\psi''(\sqrt{\sfrac{p}{\tau}}) &= e^{-\frac{\bcancel{\tau}}{2}\frac{p}{\bcancel{\tau}}} \cdot \sqrt{\frac{p}{\tau}}^{p-2} \left(\bcancel{\tau^2} \left(\frac{p}{\bcancel{\tau}}\right)^2 -2\bcancel{\tau} p \frac{p}{\bcancel{\tau}} -\bcancel{\tau} \frac{p}{\bcancel{\tau}} + p^2 -p \right)\\
&= \underset{> 0}{\underbrace{\vphantom{\int}e^{-\frac{p}{2}}}} \underset{> 0}{\underbrace{\vphantom{\int}\left(\frac{p}{\tau}\right)^{\frac{p-2}{2}}}} \underset{< 0}{\left(\underbrace{\vphantom{\int}\bcancel{p^2} - \bcancel{2p^2} -p + \bcancel{p^2} -p}\right)} < 0 & \text{ since } p\geq 2
\end{align*}
Hence, we conclude that we have found a local maximum at $t = \sqrt{\sfrac{p}{\tau}}$, therefore $t^p e^{-\frac{\tau}{2}t^2} \leq \left(\frac{p}{\tau e}\right)^\frac{p}{2}$.
\end{proof}

\begin{lemma}\label{misc2}
Let be $t>0$, $p\geq 2$, $\tau \in (0, 1)$, and $n > 0$. Then we have
\begin{align}
\left(\frac{p}{\tau e}\right)^{\frac{p}{2}} (1-\tau)^{-\frac{n}{2}} \leq \left(p+n\right)^{\frac{p}{2}}\label{eq:misc2}
\end{align}
\end{lemma}
\begin{proof}
Denote $\phi(\tau):= \left(\frac{p}{\tau e}\right)^{\frac{p}{2}} (1-\tau)^{-\frac{n}{2}}$ for some fixed $\tau\in(0,1)$. Find the point $\tau$ s.t. $\phi'(\tau) = 0$.
\begin{align*}
\phi'(\tau)=\frac{ e^{-\frac{p}{2}} (1 - \tau)^{-\frac{n}{2} - 1} (\frac{p}{\tau})^{\frac{p}{2}} (n \tau + p (\tau - 1))}{2 \tau} &\overset{!}{=} 0\\
\tau&=\frac{p}{n+p} & \text{ since }n+p \neq 0, np \neq 0\\
\end{align*}
Let us check the sign of the second derivative of $\phi(\frac{p}{n+p})$.
\begin{align*}
\phi''(\tau)&=\frac{1}{4(\tau-1)^2\tau^2} e^{-\frac{p}{2}}\left(\frac{p}{\tau}\right)^{\frac{p}{2}} (1-\tau)^{-\frac{n}{2}} \left(\tau^2((p+ n)^2 + 2(p + n))\right. \\
&\hspace{20em}\left.- 2p \tau(p  + n) - 4p\tau + 2p\right)\\
\phi''(\frac{p}{n+p})&= \underset{>0}{\underbrace{\vphantom{\int}\frac{(n+p)^4}{4n^2p^2}}} ~ \underset{>0}{\underbrace{\vphantom{\int}e^{-\frac{p}{2}}}} ~\underset{>0}{\underbrace{\vphantom{\int}(n+p)^{\frac{p}{2}}}} ~\underset{>0}{\underbrace{\vphantom{\int}\left(\frac{n+p}{n}\right)^{\frac{n}{2}}}} ~\underset{<0}{\left(\underbrace{\vphantom{\int} -p^2 \left( 1+\frac{2}{p+n}\right) +2p}\right)} < 0
\end{align*}
since $p \geq 2, n > 0$. Therefore, $\phi(\tau)$ takes a local maximum at $\tau = \frac{p}{n+p}$. Then we obtain
\begin{align*}
\phi(\frac{p}{n+p})&= e^{-\frac{p}{2}} \left(p+n\right)^{\frac{p}{2}} \left(\frac{n+p-p}{n+p}\right)^{-\frac{n}{2}} \leq \left(p+n\right)^{\frac{p}{2}} .
\end{align*}
The last inequality follows because
\begin{align*}
e^{-\frac{p}{2}}\left(\frac{n}{n+p}\right)^{-\frac{n}{2}} &\leq 1 .
\end{align*}
\end{proof}

\begin{lemma}\label{misc3}
\begin{align}
\nabla f(\w) = \E_{\ug}\left[ \langle \nabla f(\w),\ug \rangle \ug  \right] \label{eq:misc3}
\end{align}
\end{lemma}
\begin{proof}
Let us denote the gradients $\nabla f(\w)$ wrt. $\w$ as a column vector of derivatives of each component in $\w$.
\begin{align*}
\nabla f(\w) = \left[\frac{\partial f}{\partial w_1}, \frac{\partial f}{\partial w_2}, \cdots \frac{\partial f}{\partial w_n}\right]^\top
\end{align*}
By definition of the scalar product, we have
\begin{align}
\langle \nabla f(\w),\ug \rangle = \sum_{j=1}^{n}\frac{\partial f}{\partial w_j} u_j \label{eq:g1}
\end{align}
Then we get
\begin{align}
\E_{\ug}\left[ \langle \nabla f(\w),\ug \rangle \ug  \right] &\overset{\eqref{eq:g1}}{=} \E_{\ug}\left[ \left(\sum_{j=1}^{n}\frac{\partial f}{\partial w_j} u_j\right) \ug\right] \nonumber\\
&= \sum_{j=1}^{n}\E_{\ug}\left[ \left(\frac{\partial f}{\partial w_j} u_j\right) \ug\right]\label{eq:g2}\\
&= \sum_{j=1}^{n}\frac{\partial f}{\partial w_j}\E_{\ug}\left[  u_j \cdot\ug\right] \nonumber\\
&=\sum_{j=1}^{n}\frac{\partial f}{\partial w_j}
\left[ \begin{array}{c}
\E \left[u_j \cdot u_1\right] \\
\E \left[u_j \cdot u_2\right] \\
\vdots \\
\E \left[u_j \cdot u_j\right]\\
\vdots \\
\E \left[u_j \cdot u_n\right] \\
\end{array}\right] \nonumber\\
&=\sum_{j=1}^{n}\frac{\partial f}{\partial w_j}
\left[ \begin{array}{c}
\E \left[u_j\right] \cdot \E \left[u_1\right] \\
\E \left[u_j\right] \cdot \E \left[u_2\right] \\
\vdots \\
\E \left[u_j^2\right]\\
\vdots \\
\E \left[u_j\right] \cdot \E \left[u_n\right] \\
\end{array}\right] \label{eq:g3}
\end{align}
We used the linearity of expectation in \eqref{eq:g2}. The last equality \eqref{eq:g3} follows because each component $u$ is independent. Since $u_j$ was drawn from $\N(0, 1)$, we know $\E[u_j] = 0$, and $\E[(u_j - 0)^2] = 1$. Then we obtain
\begin{align}
\E_{\ug}\left[ \langle \nabla f(\w),\ug \rangle \ug \right] \overset{\eqref{eq:g3}}{=} \sum_{j=1}^{n}\frac{\partial f}{\partial w_j} 
\left[ \begin{array}{c}
0 \\
0 \\
\vdots \\
1\\
\vdots \\
0 \\
\end{array}\right] 
= \left[ \begin{array}{c}
\frac{\partial f}{\partial w_1} \\
\frac{\partial f}{\partial w_2} \\
\vdots \\
\frac{\partial f}{\partial w_j}\\
\vdots \\
\frac{\partial f}{\partial w_n} \\
\end{array}\right]
= \nabla f(\w)
\end{align}

\end{proof}

\subsection{Empirical Sparsity}\label{sparsity}

Figures \ref{fig:sparsity-cifar} and \ref{fig:sparsity-mnist} illustrate our observation of empirical gradient sparsity across the entire learning process on both MNIST and CIFAR classification tasks. The histograms on the right side show that the majority of gradient coordinates have value zero throughout the learning process, while the values of the learned weight vector are centered around non-zero values, except obviously for pruning approaches.

\begin{figure}
    \centering
\begin{tabular}{ccc}
     \rotatebox[origin=t]{90}{dense}& \makecell{\includegraphics[trim={3em 0 0 0},clip,width=0.45\textwidth]{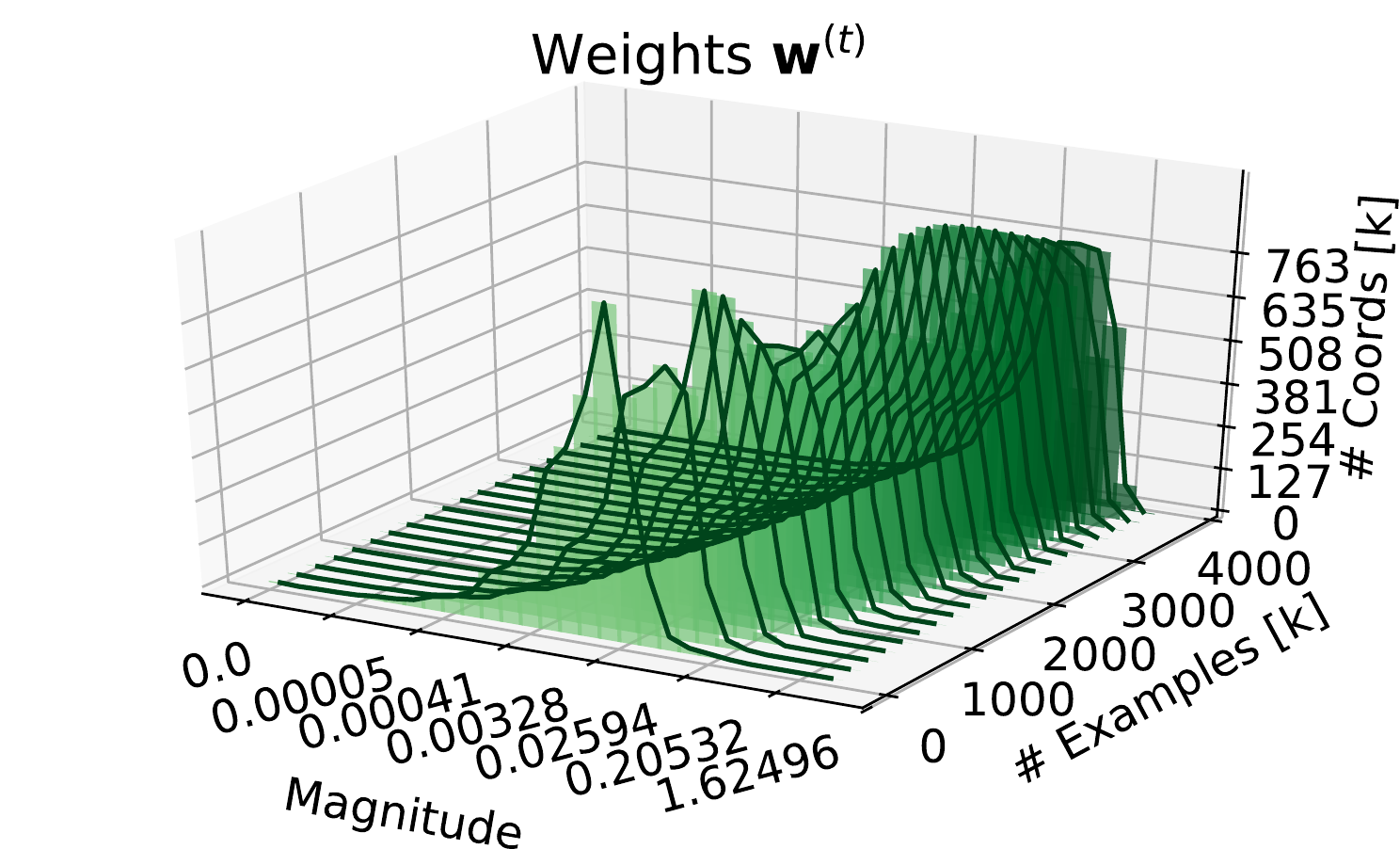}}&\makecell{\includegraphics[trim={3em 0 0 0},clip,width=0.45\textwidth]{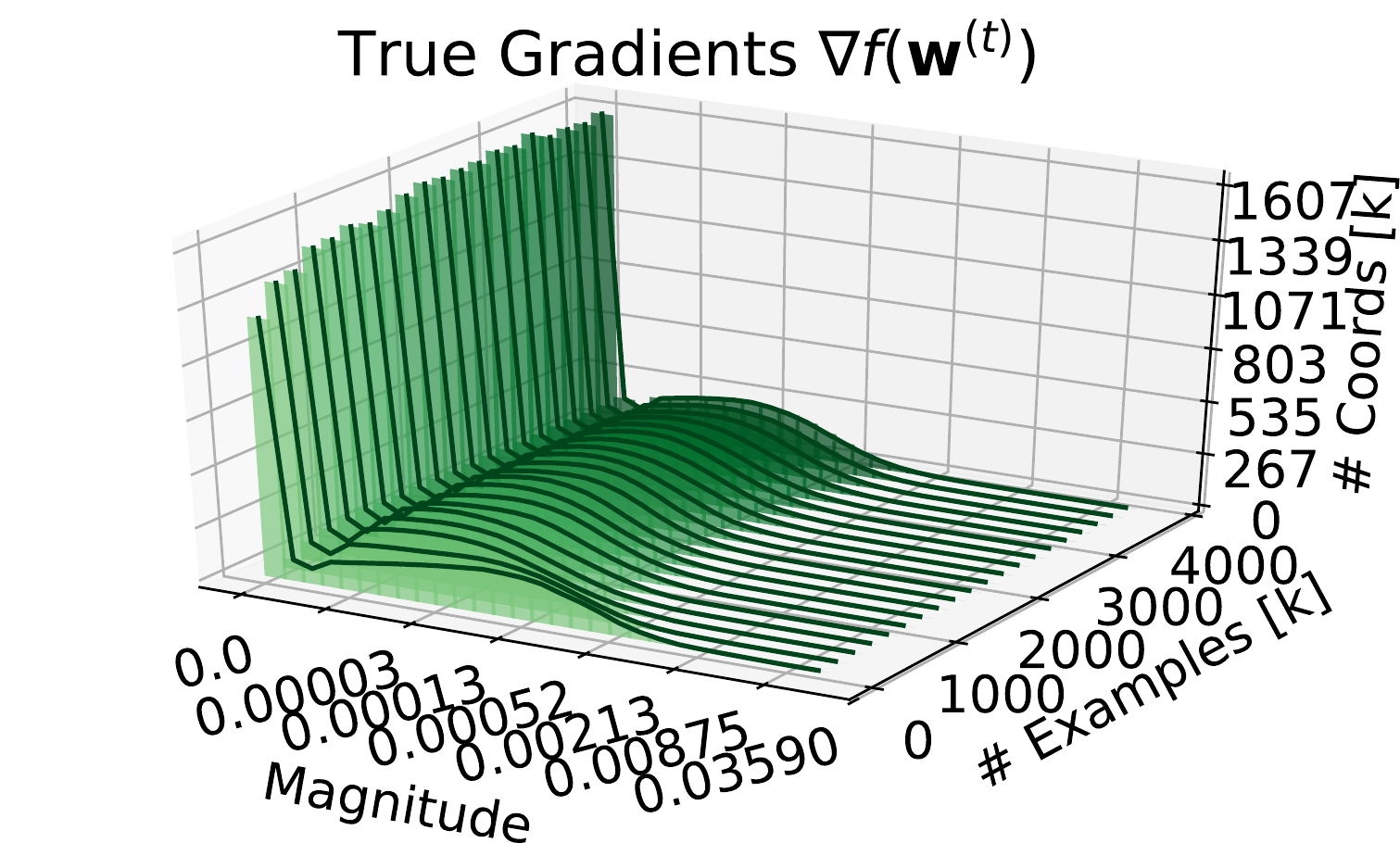}}\\
     \rotatebox[origin=t]{90}{prune-L1}& \makecell{\includegraphics[trim={3em 0 0 0},clip,width=0.45\textwidth]{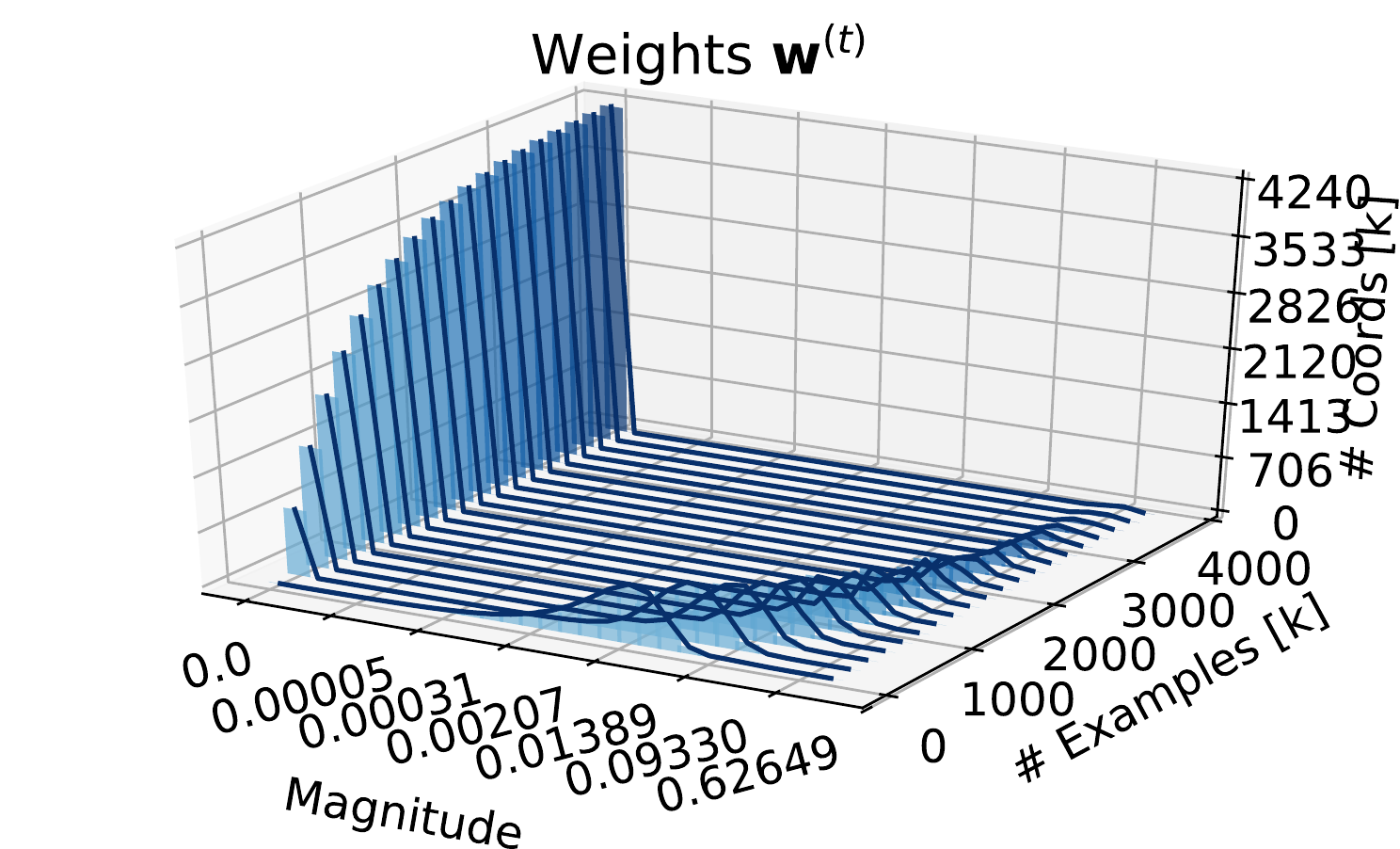}}&\makecell{\includegraphics[trim={3em 0 0 0},clip,width=0.45\textwidth]{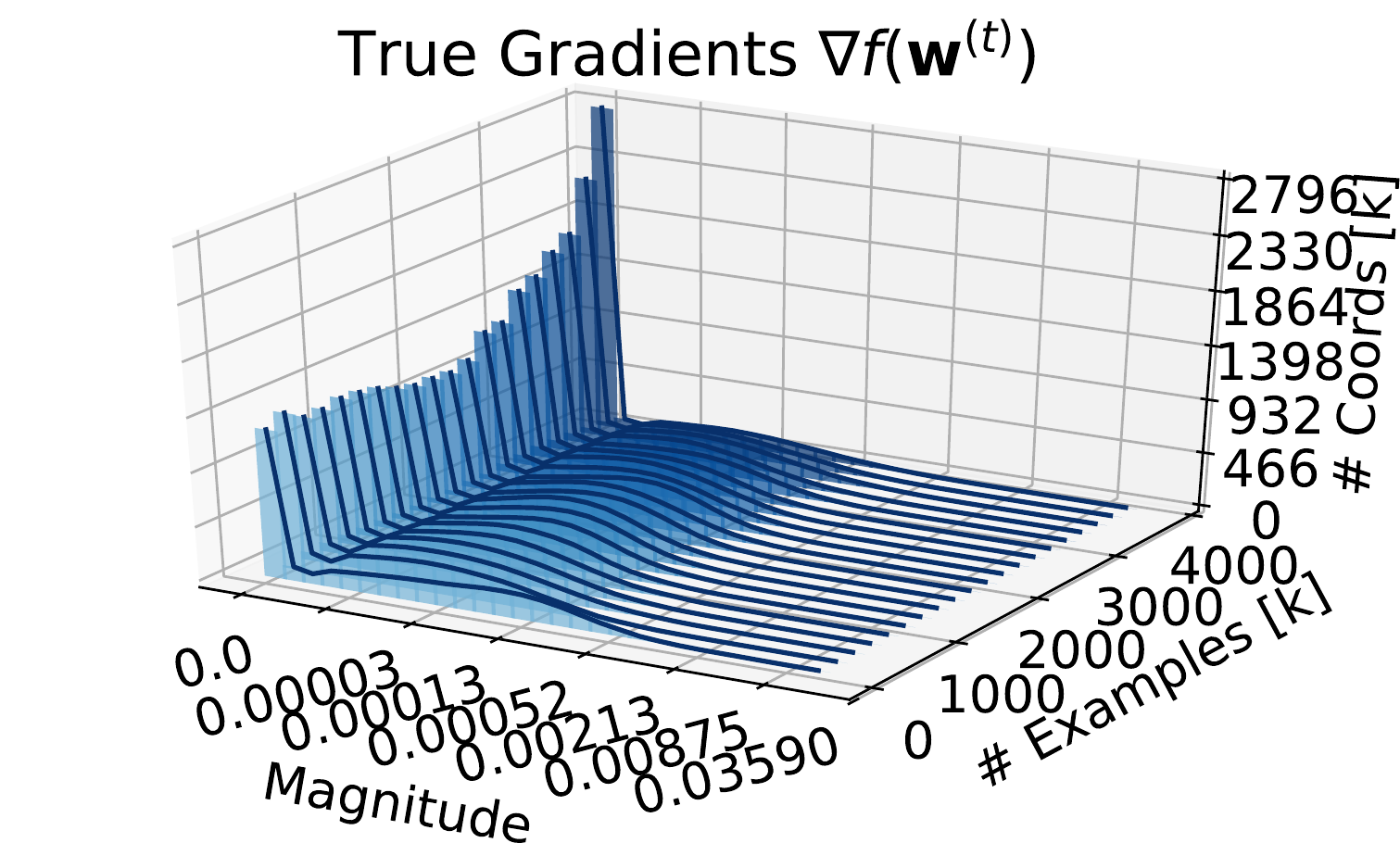}}\\
     \rotatebox[origin=t]{90}{prune-random}& \makecell{\includegraphics[trim={3em 0 0 0},clip,width=0.45\textwidth]{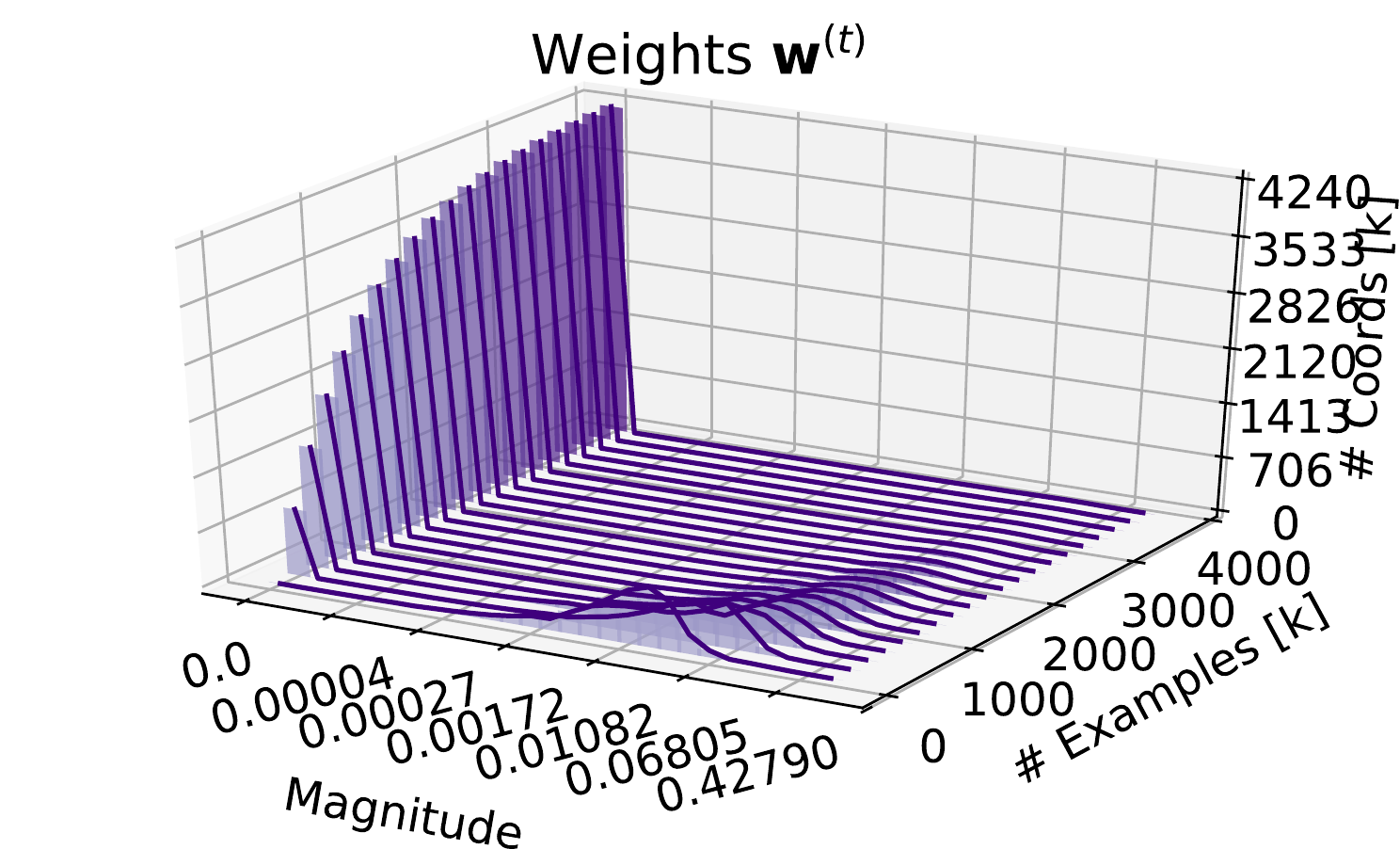}}&\makecell{\includegraphics[trim={3em 0 0 0},clip,width=0.45\textwidth]{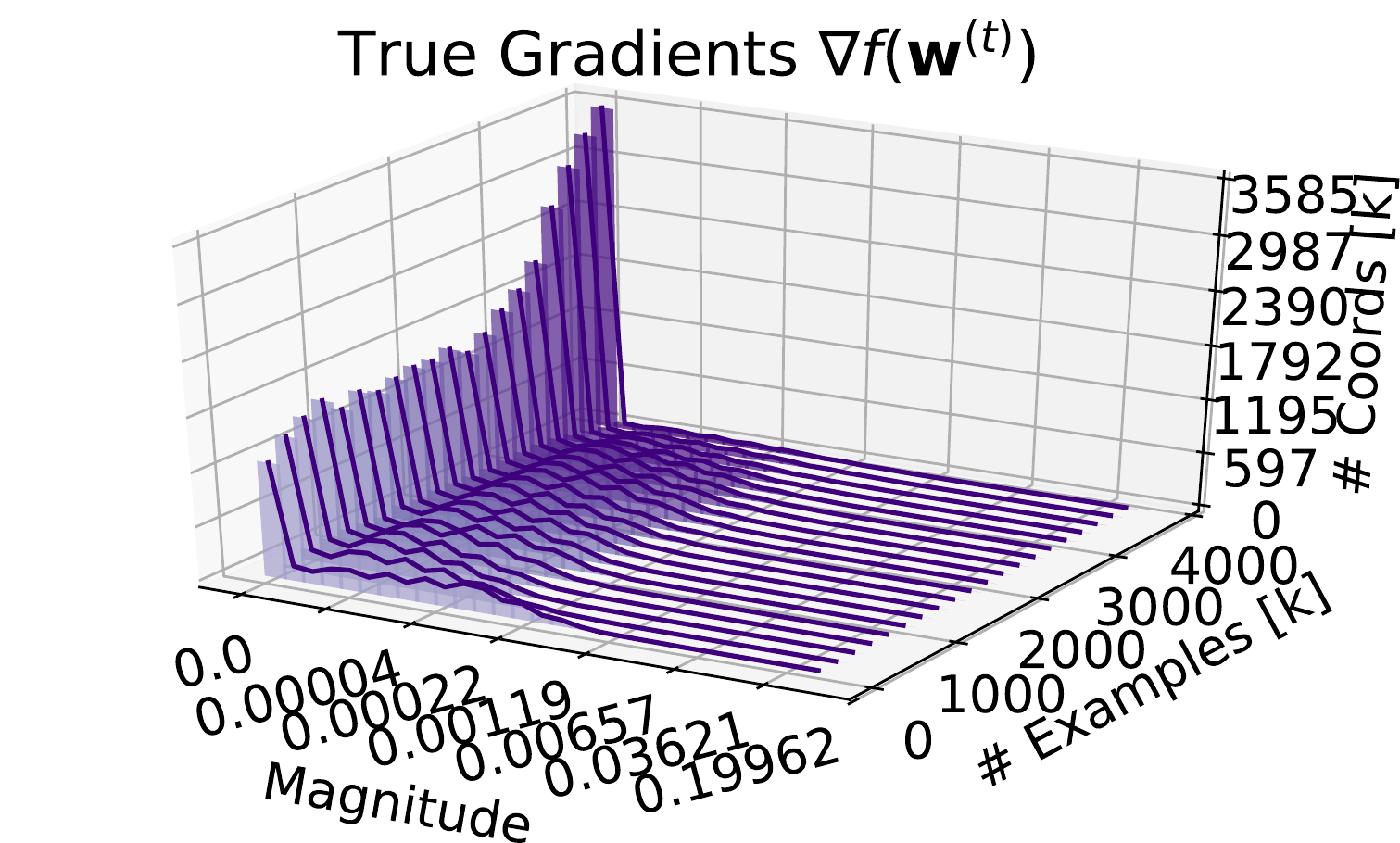}}\\
     \rotatebox[origin=t]{90}{freeze-L1}& \makecell{\includegraphics[trim={3em 0 0 0},clip,width=0.45\textwidth]{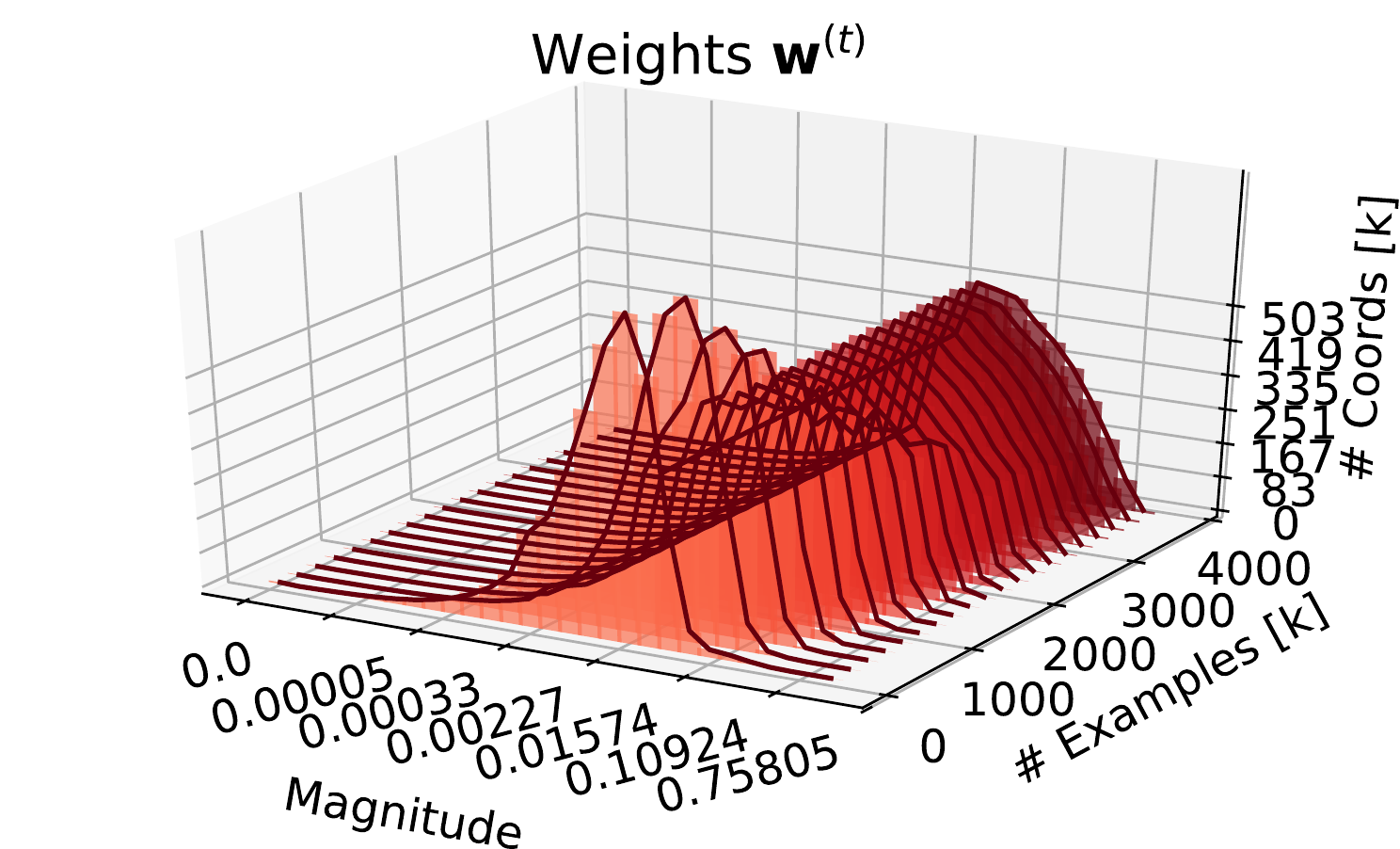}}&\makecell{\includegraphics[trim={3em 0 0 0},clip,width=0.45\textwidth]{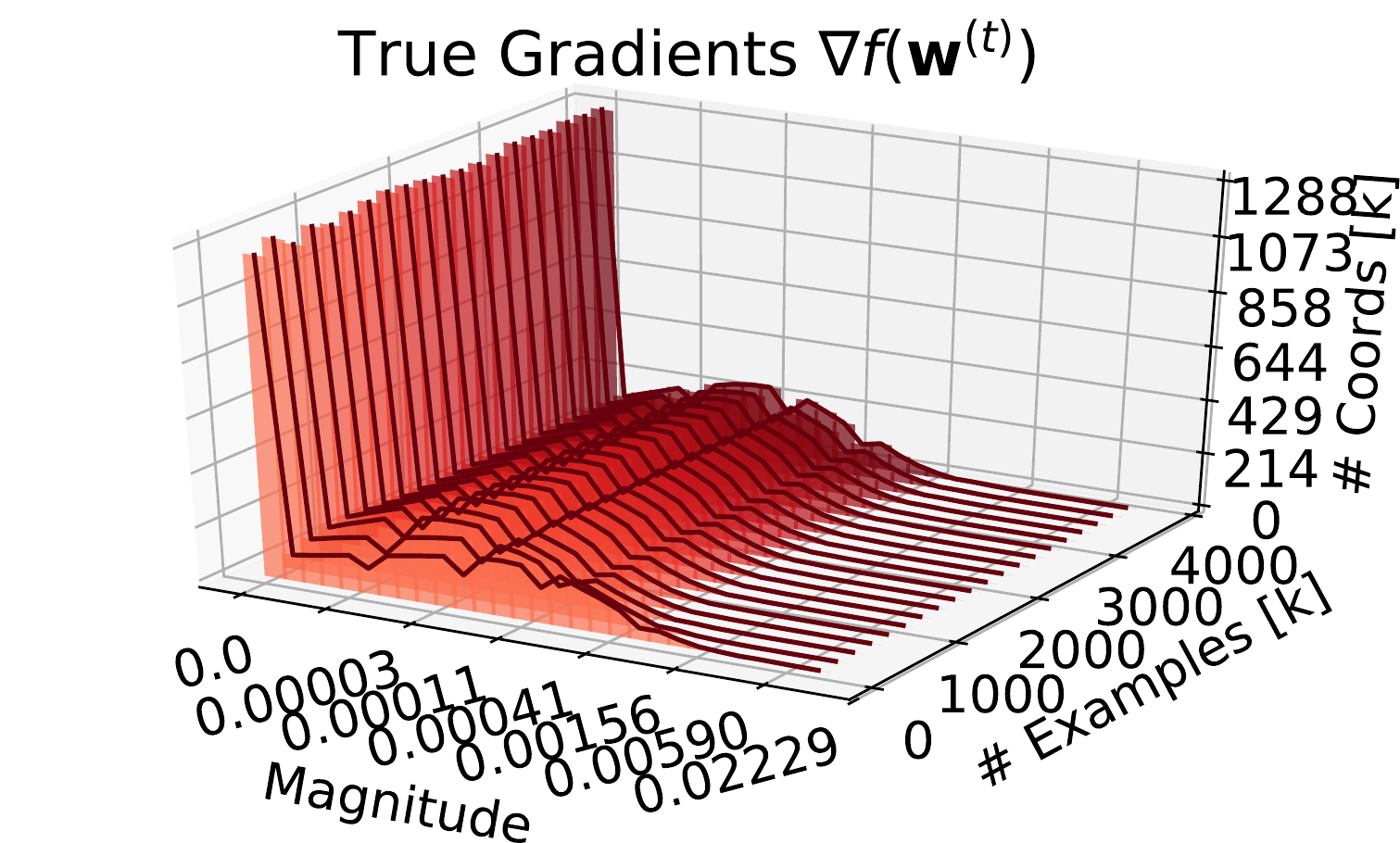}}\\
     \rotatebox[origin=t]{90}{freeze-random}& \makecell{\includegraphics[trim={3em 0 0 0},clip,width=0.45\textwidth]{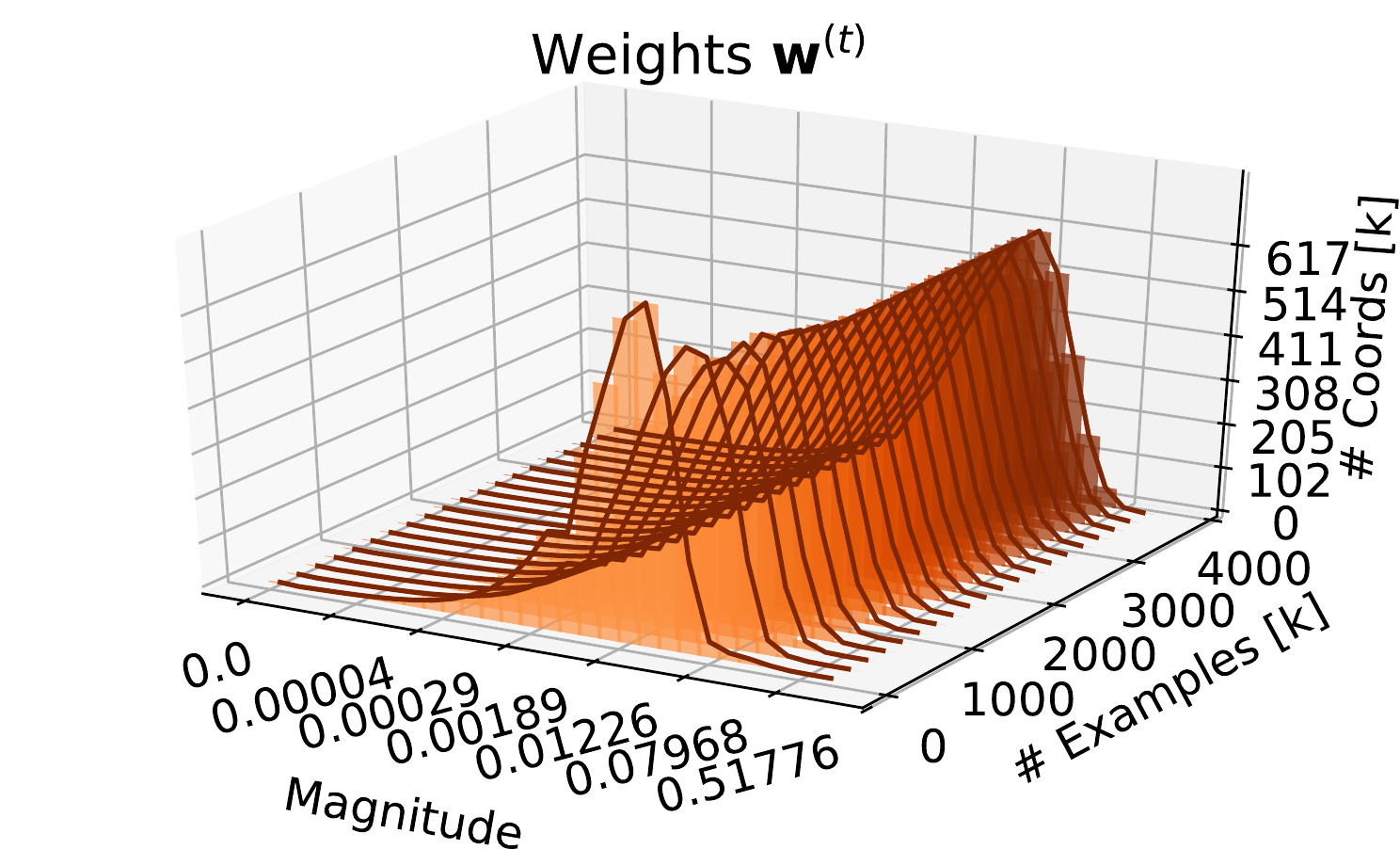}}&\makecell{\includegraphics[trim={3em 0 0 0},clip,width=0.45\textwidth]{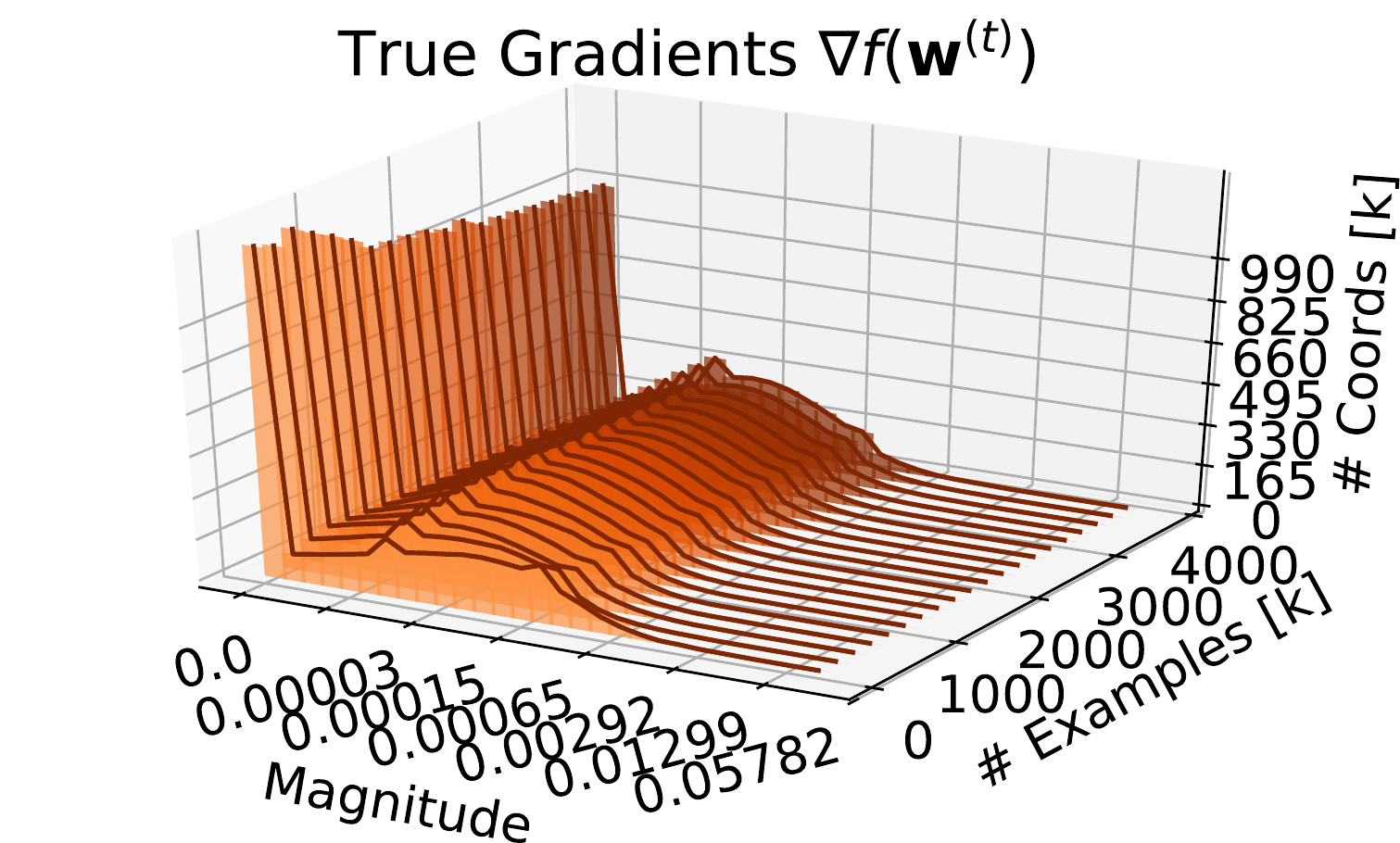}}
\end{tabular}
    \caption{Empirical gradient sparsity and values of learned weight vector for CIFAR-10.}
    \label{fig:sparsity-cifar}
\end{figure}

\begin{figure}
    \centering
\begin{tabular}{ccc}
     \rotatebox[origin=c]{90}{dense}& \makecell{\includegraphics[trim={3em 0 0 0},clip,width=0.45\textwidth]{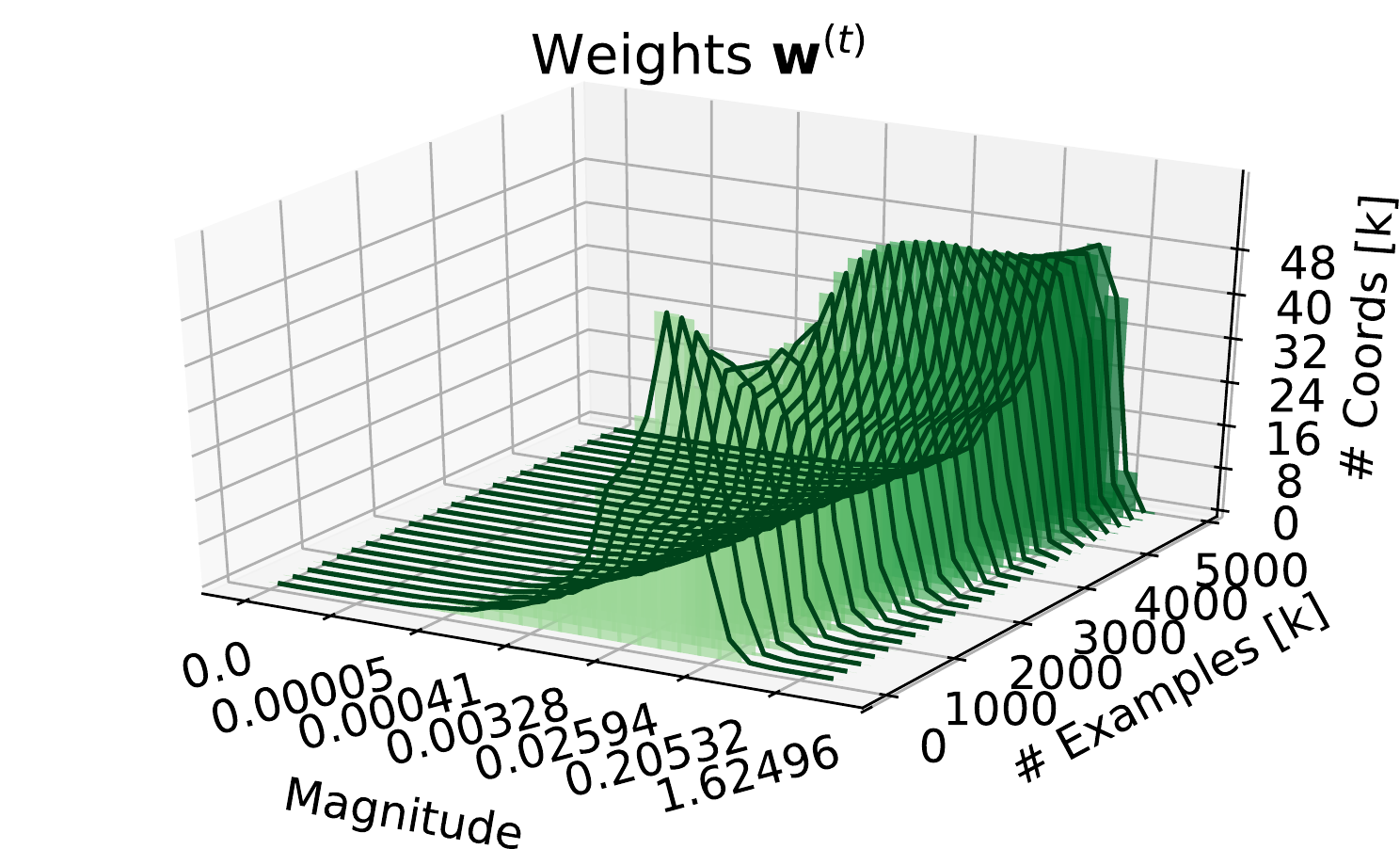}}&\makecell{\includegraphics[trim={3em 0 0 0},clip,width=0.45\textwidth]{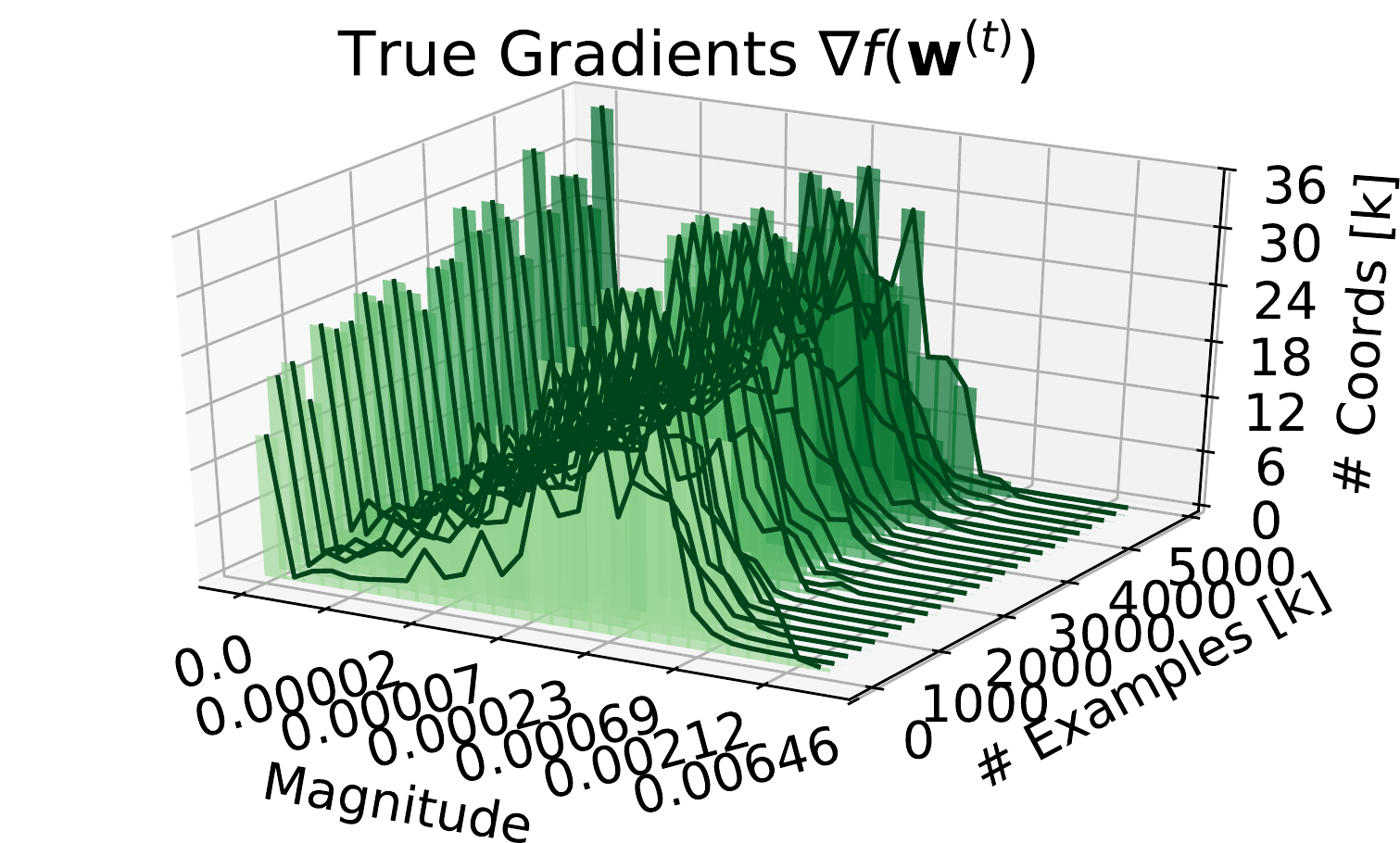}}\\
     \rotatebox[origin=t]{90}{prune-L1}& \makecell{\includegraphics[trim={3em 0 0 0},clip,width=0.45\textwidth]{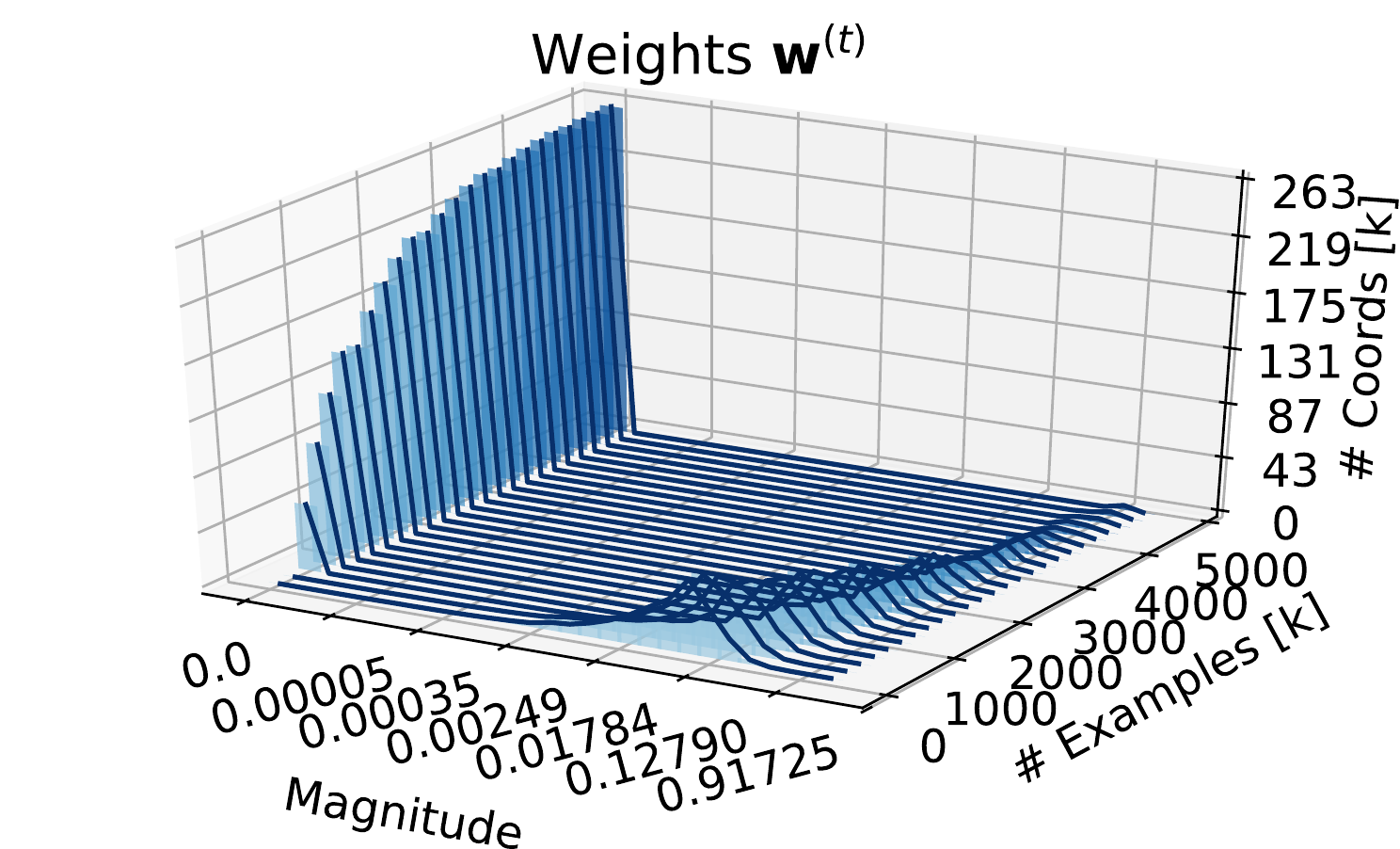}}&\makecell{\includegraphics[trim={3em 0 0 0},clip,width=0.45\textwidth]{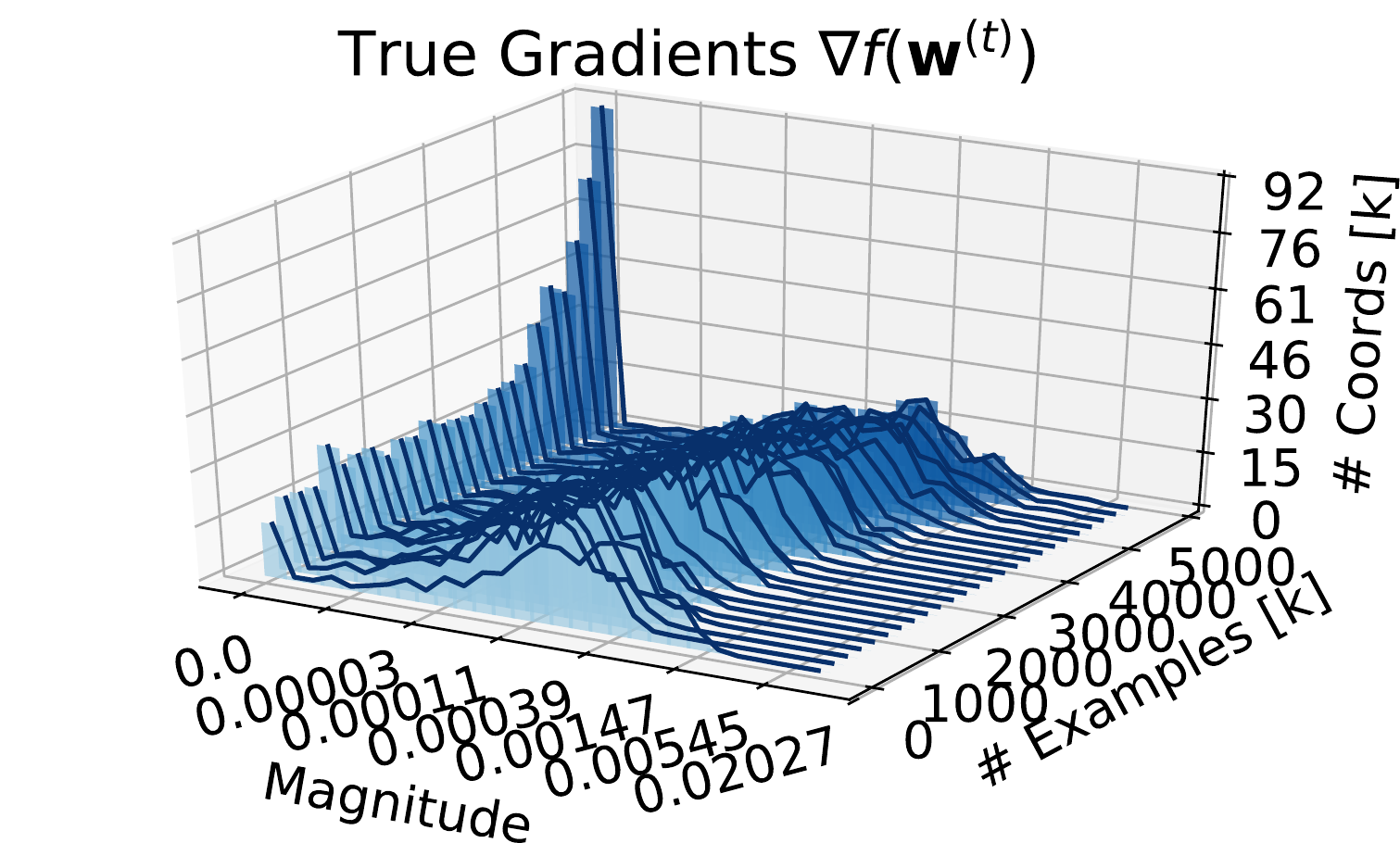}}\\
     \rotatebox[origin=t]{90}{prune-random}& \makecell{\includegraphics[trim={3em 0 0 0},clip,width=0.45\textwidth]{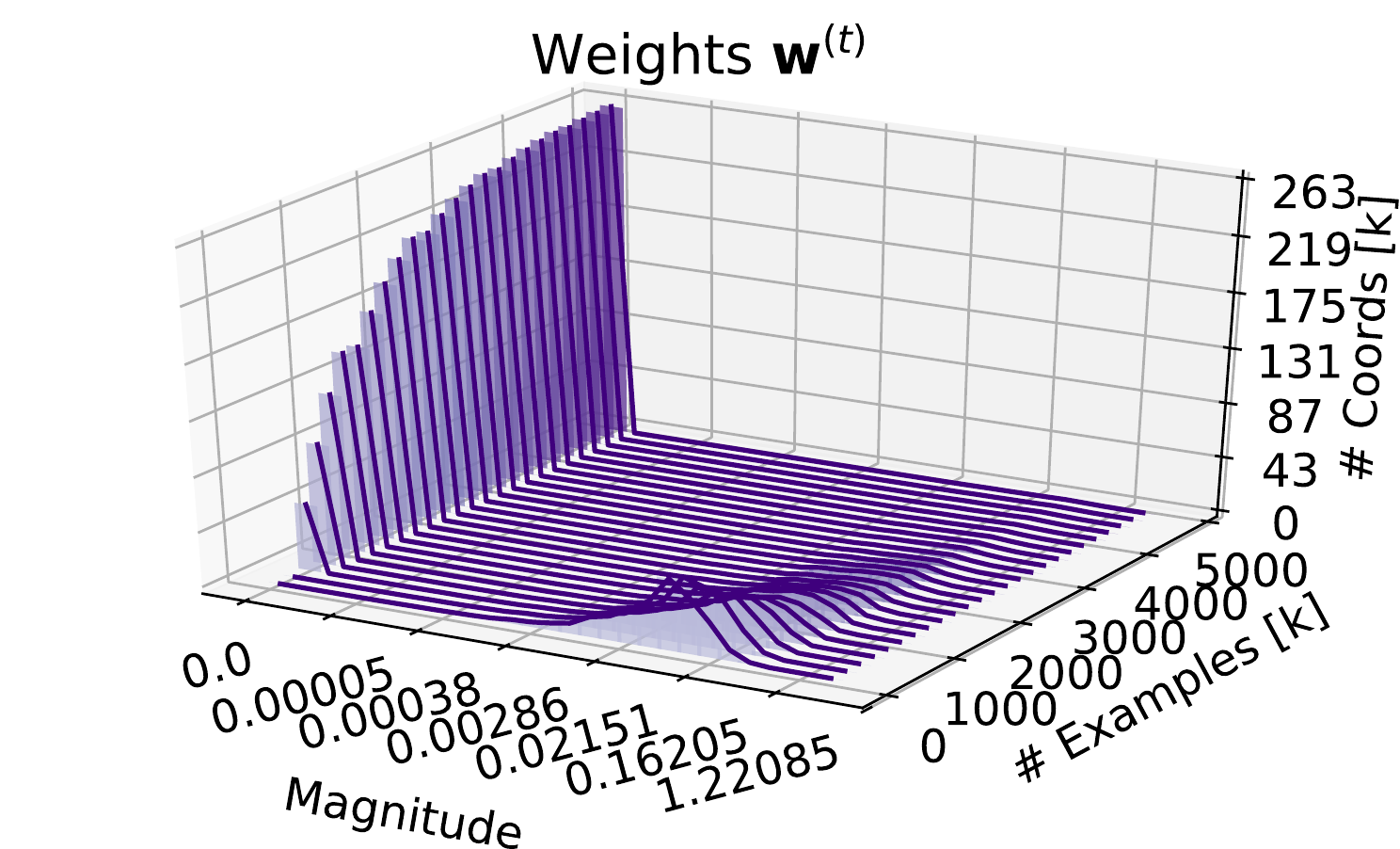}}&\makecell{\includegraphics[trim={3em 0 0 0},clip,width=0.45\textwidth]{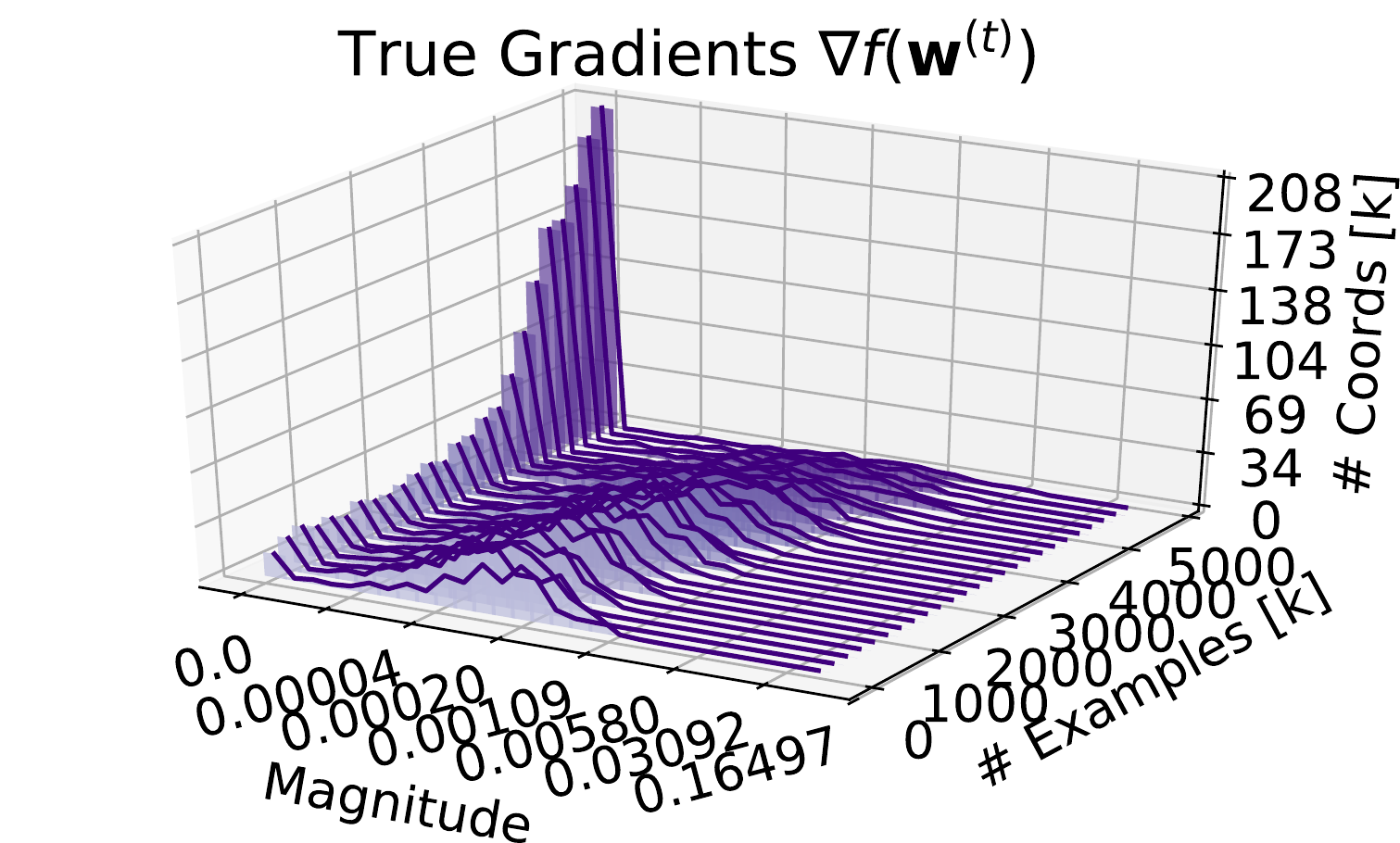}}\\
     \rotatebox[origin=t]{90}{freeze-L1}& \makecell{\includegraphics[trim={3em 0 0 0},clip,width=0.45\textwidth]{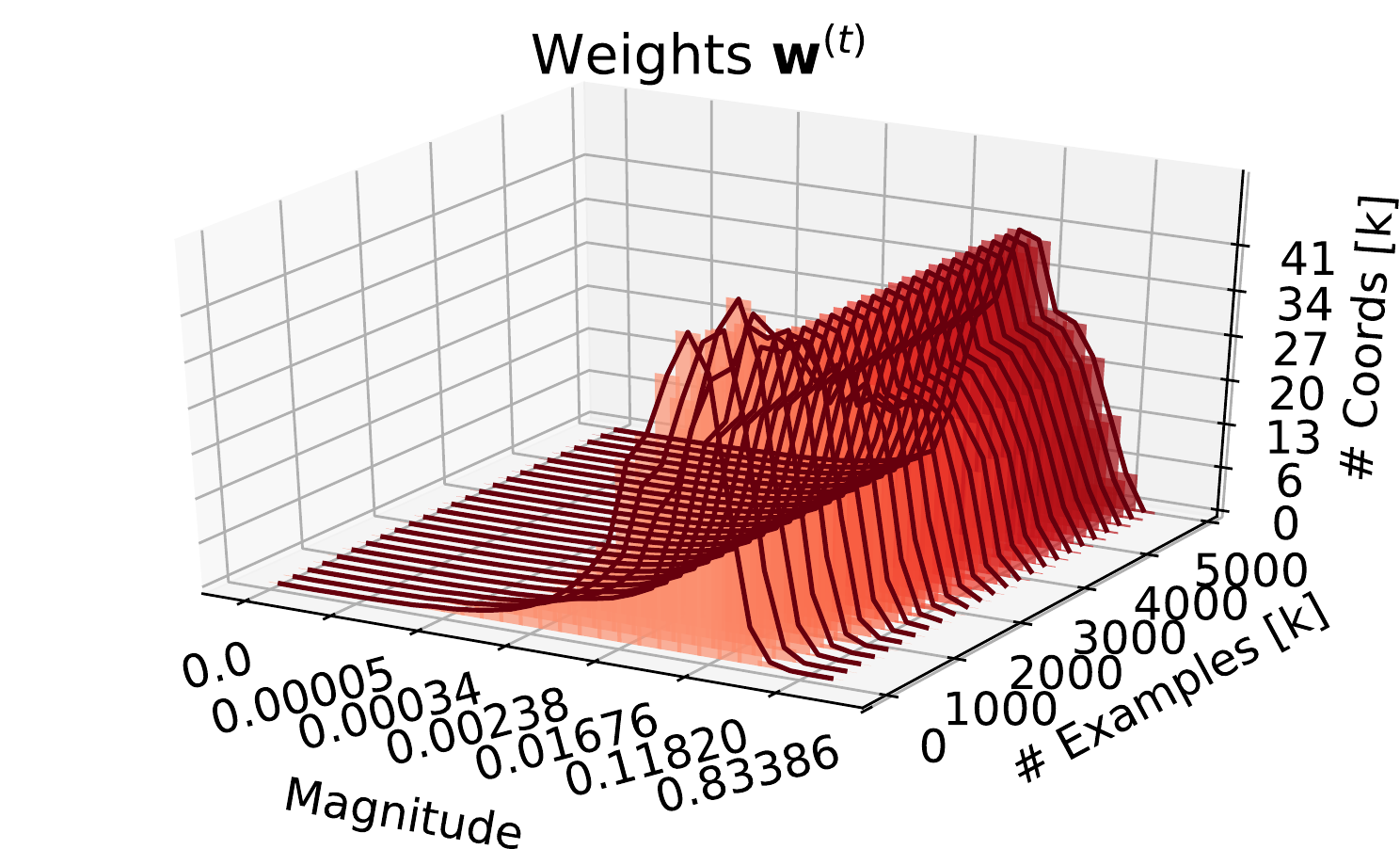}}&\makecell{\includegraphics[trim={3em 0 0 0},clip,width=0.45\textwidth]{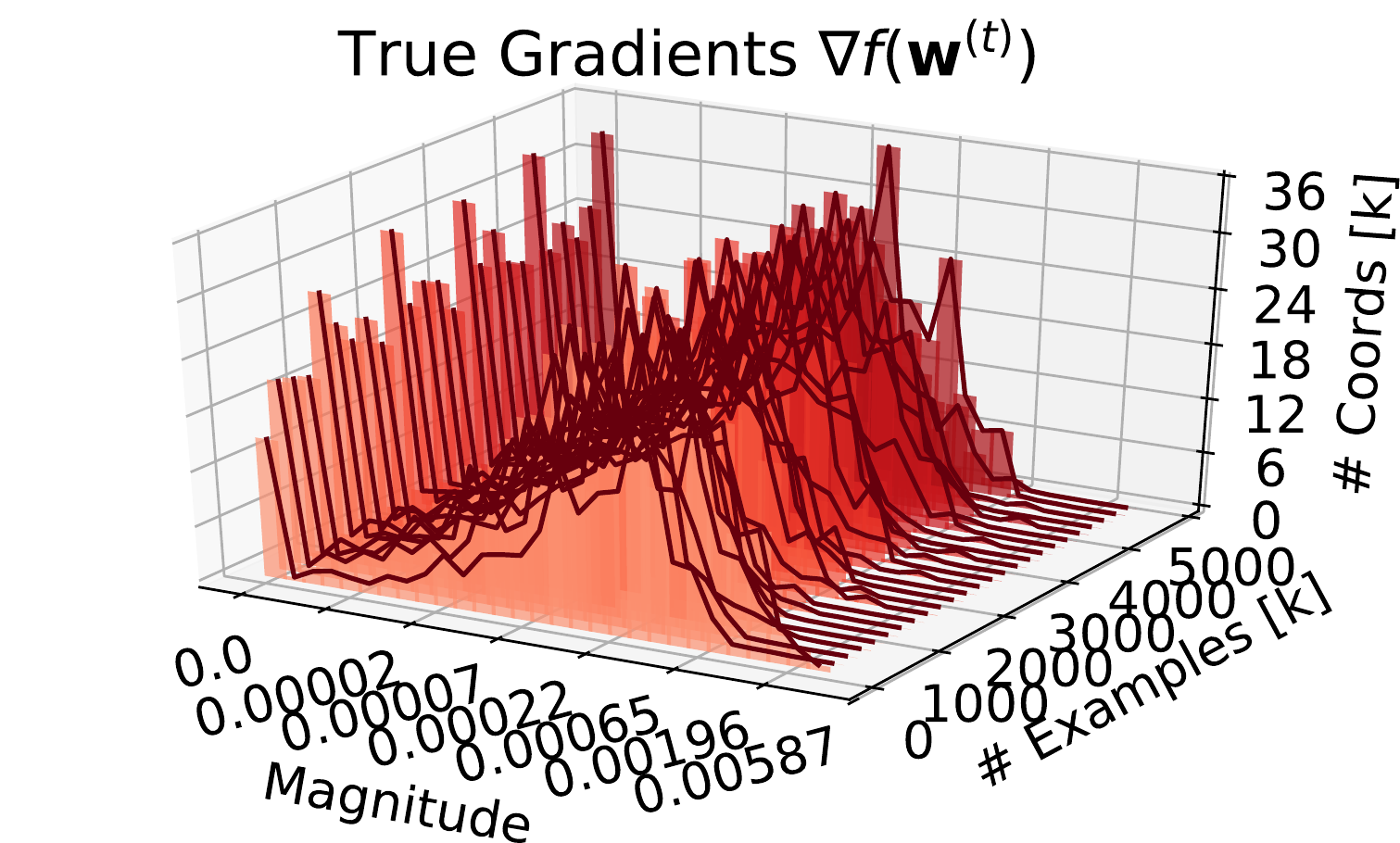}}\\
     \rotatebox[origin=t]{90}{freeze-random}& \makecell{\includegraphics[trim={3em 0 0 0},clip,width=0.45\textwidth]{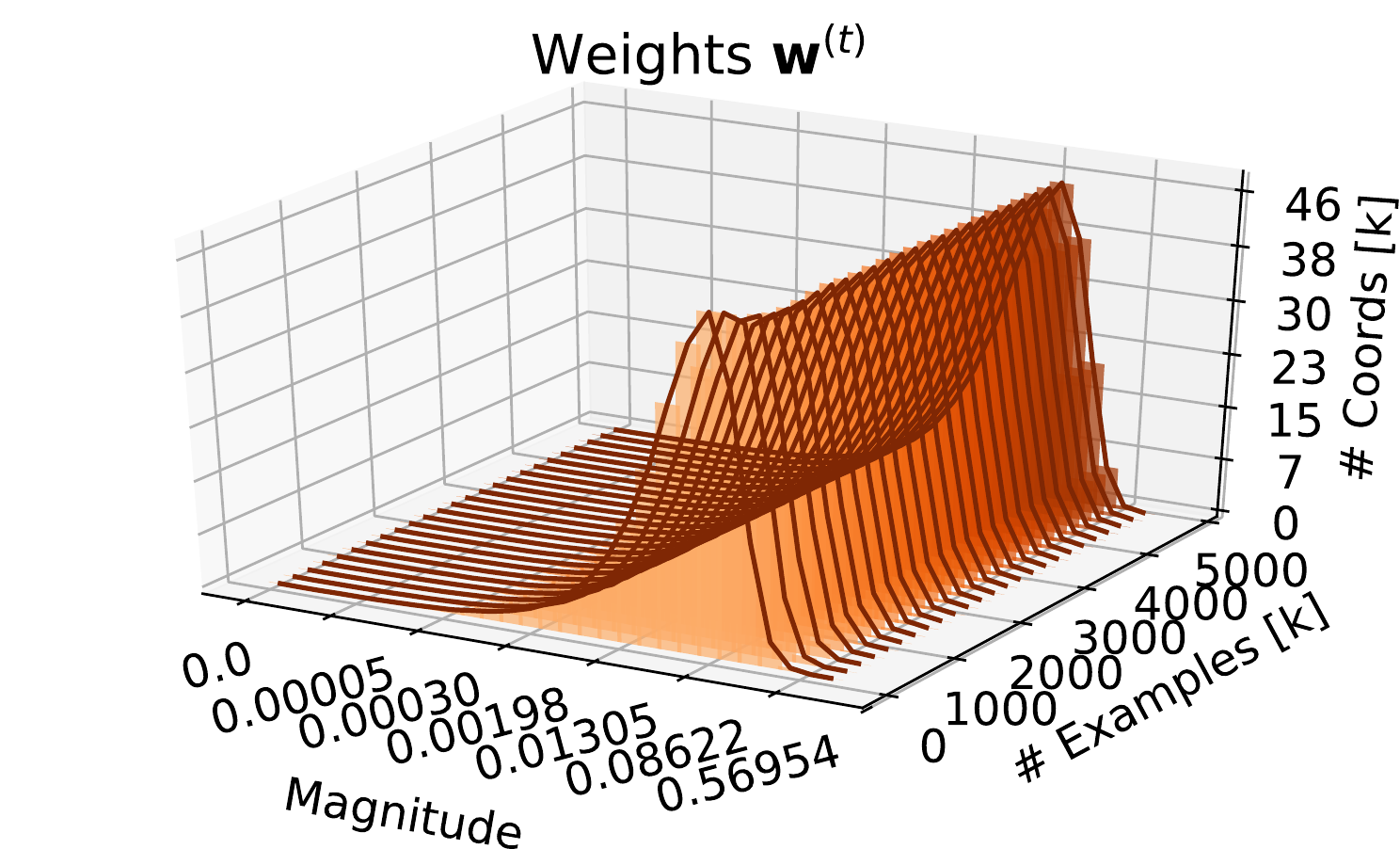}}&\makecell{\includegraphics[trim={3em 0 0 0},clip,width=0.45\textwidth]{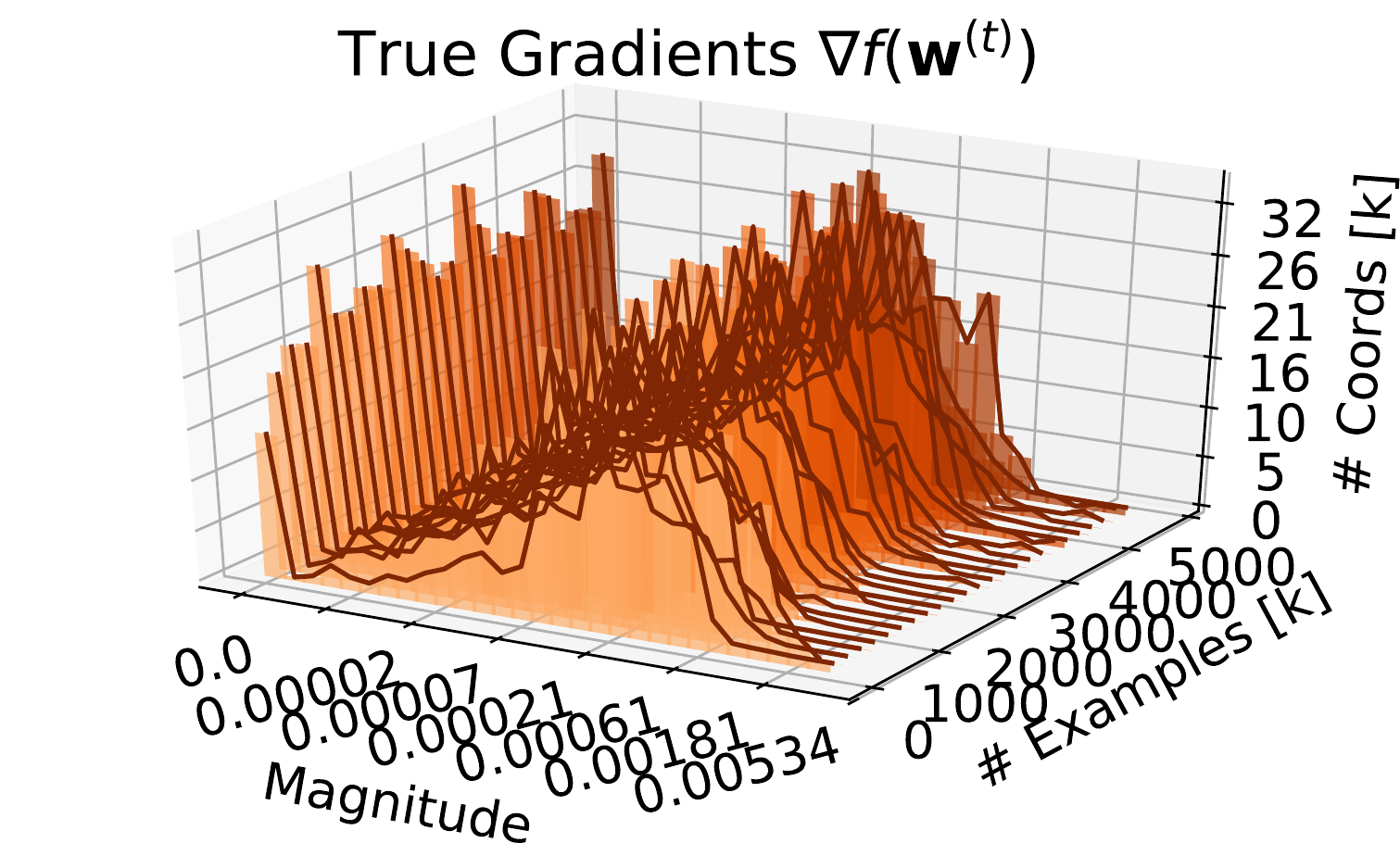}}
\end{tabular}
    \caption{Empirical gradient sparsity and values of learned weight vector for MNIST.}
    \label{fig:sparsity-mnist}
\end{figure}

\end{document}